\documentclass{article} 
\usepackage{iclr2026_conference,times}

\usepackage{amsmath,amsfonts,bm}

\def\eqref#1{equation~\ref{#1}}

\def\1{\bm{1}}

\DeclareMathAlphabet{\mathsfit}{\encodingdefault}{\sfdefault}{m}{sl}
\SetMathAlphabet{\mathsfit}{bold}{\encodingdefault}{\sfdefault}{bx}{n}
\newcommand{\tens}[1]{\bm{\mathsfit{#1}}}

\usepackage{hyperref}
\usepackage{url}
\usepackage{graphicx}
\usepackage{wrapfig}
\usepackage{booktabs}
\usepackage{amsmath}
\usepackage{amssymb}
\usepackage{amsthm}
\usepackage{algorithm}
\usepackage{algorithmic}
\usepackage{multirow}
\usepackage{float}
\usepackage{appendix}
\usepackage{titletoc}
\usepackage{enumitem}
\usepackage{subfigure}
\usepackage{subcaption}
\usepackage{tabularx}

\usepackage[table]{xcolor}
\definecolor{citecolor}{HTML}{0071BC}
\definecolor{linkcolor}{HTML}{ED1C24}
\definecolor{mydarkblue}{rgb}{0,0.08,0.45}
\hypersetup{colorlinks=true, linkcolor=mydarkblue, citecolor=mydarkblue,urlcolor=mydarkblue}
\definecolor{LightCyan}{rgb}{0.4, 0.5, 1}
\definecolor{codeblue}{rgb}{0.58, 0.94, 0.85}

\theoremstyle{plain}
\newtheorem{theorem}{Theorem}
\newtheorem{assumption}{Assumption}
\newtheorem{remark}{Remark}

\title{Multi-Head Low-Rank Attention}

\author{Songtao Liu\textsuperscript{1}\thanks{Correspondence to: Songtao Liu $<$skl5761@psu.edu$>$.}~~~~Hongwu Peng\textsuperscript{2}~~~~Zhiwei Zhang\textsuperscript{1}~~~~Zhengyu Chen\textsuperscript{3}~~~~Yue Guo\textsuperscript{4} \\
\textsuperscript{1}The Pennsylvania State University~~~~
\textsuperscript{2}University of Connecticut\\
\textsuperscript{3}Carnegie Mellon University~~~~
\textsuperscript{4}University of California, Los Angeles
}

\iclrfinalcopy 
\begin{document}

\maketitle

\begin{abstract}
Long-context inference in large language models is bottlenecked by Key--Value (KV) cache loading during the decoding stage, where the sequential nature of generation requires repeatedly transferring the KV cache from off-chip High-Bandwidth Memory (HBM) to on-chip Static Random-Access Memory (SRAM) at each step. While Multi-Head Latent Attention (MLA) significantly reduces the total KV cache size, it suffers from a sharding bottleneck during distributed decoding via Tensor Parallelism (TP). Since its single latent head cannot be partitioned, each device is forced to redundantly load the complete KV cache for every token, consuming excessive memory traffic and diminishing TP benefits like weight sharding. In this work, we propose Multi-Head Low-Rank Attention (MLRA), which enables partitionable latent states for efficient 4-way TP decoding. Extensive experiments show that MLRA achieves state-of-the-art perplexity and downstream task performance, while also delivering a 2.8$\times$ decoding speedup over MLA. Code is available at \url{https://github.com/SongtaoLiu0823/MLRA}. Pretrained weights, along with the training and evaluation data, are available at \url{https://huggingface.co/Soughing/MLRA}.
\end{abstract}
\section{Introduction}
Inference-time scaling~\citep{jaech2024openai} is critical for large language models (LLMs) to produce high-quality responses. Both retrieval-augmented generation (RAG)~\citep{lewis2020retrieval} and long chain-of-thought (CoT) reasoning~\citep{wei2022chain} rely on maintaining long context before generating the final answer, substantially increasing the number of tokens that must be processed at each decoding step. Sequential token generation under Multi-Head Attention (MHA)~\citep{vaswani2017attention} requires reloading the Key--Value (KV) cache from high-bandwidth memory every step, so data movement~\citep{ivanov2021data,gholami2024ai}, not computation, dominates latency for long-context inference~\citep{sadhukhan2025magicdec}. The small amount of compute per step relative to this data movement leads to poor GPU utilization~\citep{he2024fastdecode,zadouri2025hardware}.

A series of recent studies~\citep{shazeer2019fast,hu2024multi,zadouri2025hardware,zhang2025tensor,zheng2025sas} have developed alternative attention mechanisms aimed at improving decoding efficiency and overall model quality. Multi-Head Latent Attention (MLA)~\citep{liu2024deepseek} compresses the KV cache into a latent head ($4.5 d_h$ per token). By absorbing the up-projection matrices into the queries during decoding, it delivers better efficiency compared with MHA. However, MLA is unfriendly to tensor parallelism (TP) because its single latent head cannot be sharded. In this work, we address the limitation that MLA does not support TP.

We first show that partitioning the MLA latent head and the NoPE~\citep{yang2025rope} KV up-projection matrices into four blocks makes the NoPE key and value equivalent to the sum of four block-wise projections. Motivated by this insight, we propose Multi-Head Low-Rank Attention (MLRA), which explicitly decomposes the latent head into four latent heads, independently up-projects each latent head to form NoPE KV, and sums the resulting attention outputs. This design naturally supports 4-way TP and reduces the per-device KV cache loading. Based on our 2.9B scale experiments, MLRA-4 achieves the lowest perplexity (13.672 vs. 13.727 for MLA and 14.139 for GQA) and highest zero-shot common-sense reasoning accuracy (58.84\% vs. 58.75\% for MLA and 57.89\% for GQA). Our kernel delivers a 1.05-1.26$\times$ speedup over GQA in long-context decoding.
\section{Background}
\subsection{Multi-Head Latent Attention}
All notation used in this paper is summarized in Appendix~\ref{appendix:notation}. Given a sequence of $n$ tokens with hidden states $\bm{H} \in \mathbb{R}^{n \times d}$, MLA derives the query and KV states as follows:
\begin{equation*}
\begin{split}
\bm{C}^{\text{Q}} = \operatorname{RMSNorm}\left(\bm{H} \bm{W}^{\text{DQ}}\right), \quad &\bm{W}^{\text{DQ}} \in \mathbb{R}^{d \times d_c^\prime},\\
\bm{C}^{\text{KV}} = \operatorname{RMSNorm}\left(\bm{H} \bm{W}^{\text{DKV}}\right), \quad
&\bm{W}^{\text{DKV}} \in \mathbb{R}^{d \times d_c},\\
\bm{K}^{\text{RoPE}} = \operatorname{RoPE}\left(\bm{H} \bm{W}^{\text{KR}}\right), \quad
&\bm{W}^{\text{KR}} \in \mathbb{R}^{d \times d_h^R},
\end{split}
\end{equation*}
where $d_c, d_c^{\prime} \ll h d_h$ denote the dimensions for KV and query latent states, respectively. The learnable down-projection matrices $\bm{W}^{\text{DQ}}$ and $\bm{W}^{\text{DKV}}$ produce the latent states $\bm{C}^{\text{Q}}$ and $\bm{C}^{\text{KV}}$, while $\bm{W}^{\text{KR}}$ generates the partial RoPE~\citep{su2024roformer} key, denoted as $\bm{K}^{\text{RoPE}}$. Following DeepSeek-V3~\citep{deepseekv3-2024}, we set $d_c = 4d_h$ and $d_h^R = 0.5d_h$ without loss of generality. Both $\bm{C}^{\text{KV}}$ and $\bm{K}^{\text{RoPE}}$ are cached during inference to optimize efficiency. Finally, MLA derives the $h$ attention heads for queries, NoPE keys, and values through the following up-projections:
\begin{equation*}
\begin{split}
\overline{\bm{Q}}^{\text{NoPE}}=\bm{C}^{\text{Q}} \bm{W}^{\text{UQ}}, \quad \overline{\bm{Q}}^{\text{RoPE}}=\operatorname{RoPE}\left(\bm{C}^{\text{Q}} \bm{W}^{\text{QR}}\right),& \quad \bm{W}^{\text{UQ}} \in \mathbb{R}^{d_c^{\prime} \times\left(h d_h\right)}, \quad \bm{W}^{\text{QR}} \in \mathbb{R}^{d_c^{\prime} \times\left(hd_h^R \right)}\\
\overline{\bm{K}}^{\text{NoPE}} = \bm{C}^{\text{KV}} \bm{W}^{\text{UK}},\quad
\overline{\bm{V}} = \bm{C}^{\text{KV}} \bm{W}^{\text{UV}},& \quad
\bm{W}^{\text{UK}}, \bm{W}^{\text{UV}} \in \mathbb{R}^{d_c \times (h d_h)},
\end{split}
\end{equation*}
where $\bm{W}^{\text{UQ}}$, $\bm{W}^{\text{UK}}$, and $\bm{W}^{\text{UV}}$ denote the learnable up-projection matrices. To facilitate multi-head computation, the resulting queries, NoPE keys, and values are reshaped into head-wise tensors:
\begin{equation*}
\begin{split}
\tens{Q}^{\text{NoPE}} &=
\operatorname{Reshape}\left(\overline{\bm{Q}}^{\text{NoPE}}, \, \left[n, \, h, \, d_h\right]\right), \quad \tens{Q}^{\text{RoPE}} =
\operatorname{Reshape}\left(\overline{\bm{Q}}^{\text{RoPE}}, \, \left[n, \, h, \, d_h^R\right]\right)\\
\tens{K}^{\text{NoPE}} &=
\operatorname{Reshape}\left(\overline{\bm{K}}^{\text{NoPE}}, \, \left[n, \, h, \, d_h\right]\right), \quad
\tens{V} =
\operatorname{Reshape}\left(\overline{\bm{V}}, \, \left[n, \, h, \, d_h\right]\right),
\end{split}
\end{equation*}
where $\tens{Q}^{\text{NoPE}}, \tens{K}^{\text{NoPE}}, \tens{V}\in\mathbb{R}^{n\times h\times d_h}$ and $\tens{Q}^{\text{RoPE}}\in\mathbb{R}^{n\times h\times d_h^R}$.
For head $i\in\{0,\dots,h-1\}$, we define the head-specific 2D slices by indexing into the respective 3D tensors:
\begin{equation*}
\tens{Q}^{\text{NoPE}}_{:, i, :} := \tens{Q}^{\text{NoPE}}\left[:, \, i, \, :\right], \
\tens{Q}^{\text{RoPE}}_{:, i, :} := \tens{Q}^{\text{RoPE}}\left[:, \, i, \, :\right], \
\tens{K}^{\text{NoPE}}_{:, i, :}  := \tens{K}^{\text{NoPE}}\left[:, \, i, \, :\right], \
\tens{V}_{:, i, :} := \tens{V}\left[:, \, i, \, :\right].
\end{equation*}
To incorporate positional information, MLA shares a common RoPE key $\bm{K}^{\text{RoPE}}$ across all attention heads. The final position-aware query and key for head $i$ are formed by concatenating their respective NoPE and RoPE components:
\begin{equation*}
\tens{Q}_{:, i, :} = \operatorname{Concat}\left(\left[\tens{Q}_{:, i, :}^{\text{NoPE}}, \, \tens{Q}_{:, i, :}^{\text{RoPE}}\right], \, \text{dim=1}\right),\quad
\tens{K}_{:, i, :} = \operatorname{Concat}\left(\left[\tens{K}_{:, i, :}^{\text{NoPE}}, \, \bm{K}^{\text{RoPE}}\right], \, \text{dim=1}\right).
\end{equation*}
\paragraph{Efficient Decoding.}
\label{appendix:mla_decoding} 
Recall that MLA utilizes up-projection matrices $\bm{W}^{\text{UK}}$ and $\bm{W}^{\text{UV}}$. We extract the $d_h$-column slices for head $i$ to define its head-specific projections:
\begin{equation*}
\bm{W}_{:, (i)}^{\text{UK}} := \bm{W}^{\text{UK}}\left[:, \, i d_h : (i+1) d_h\right], \quad
\bm{W}_{:, (i)}^{\text{UV}} := \bm{W}^{\text{UV}}\left[:, \, i d_h : (i+1) d_h\right].
\end{equation*}
We refer to the DeepSeek official inference implementation~\citep{deepseekv3-2024} to illustrate how to ``absorb'' up-projection matrices into the queries to avoid explicit KV materialization in MLA decoding. For the prefix sequence $\{0, \ldots, n-1\}$ with cached components $\bm{C}^{\text{KV}}$ and $\bm{K}^{\text{RoPE}}$, we define the head-wise up-projections by partitioning $\bm{W}^{\text{UK}}$ and $\bm{W}^{\text{UV}}$ into $h$ heads, $\{\bm{W}_{:,(i)}^{\text{UK}}\}_{i=0}^{h-1}$ and $\{\bm{W}_{:,(i)}^{\text{UV}}\}_{i=0}^{h-1}$. For the last prefix token at position $n-1$, let $\tens{Q}_{n-1, i, :}$ denote the query vector for the $i$-th attention head. To maintain variance during the dot-product operation, we apply the scaling factor $\tau = \frac{1}{\sqrt{d_h + d_h^R}}$ and compute the attention output for the token at position $n-1$ as follows:
\begin{equation*}
\begin{split}
\tens{O}_{n-1, i, :}
=&\operatorname{Softmax}\left(
\tau\tens{Q}_{n-1, i, :}^{\text{NoPE}}\left(\bm{C}^{\text{KV}}\bm{W}_{:, (i)}^{\text{UK}}\right)^{\top}
+\tau\tens{Q}_{n-1, i, :}^{\text{RoPE}}\left(\bm{K}^{\text{RoPE}}\right)^{\top}
\right)\left(\bm{C}^{\text{KV}}\bm{W}_{:, (i)}^{\text{UV}}\right),\\
=&\ \operatorname{Softmax}\left(
\tau \underbrace{\tens{Q}_{n-1, i, :}^{\text{NoPE}}\left(\bm{W}_{:, (i)}^{\text{UK}}\right)^{\top}}_{\tilde{\tens{Q}}_{n-1, i, :}^{\text{NoPE}}\in \mathbb{R}^{d_c}}
\left(\bm{C}^{\text{KV}}\right)^{\top}
+\tau \tens{Q}_{n-1, i, :}^{\text{RoPE}} \left(\bm{K}^{\text{RoPE}}\right)^{\top}
\right)\bm{C}^{\text{KV}}\bm{W}_{:, (i)}^{\text{UV}},
\end{split}
\end{equation*}
where $\tens{O}_{n-1, i, :}$ is the attention output for head $i$ at position $n-1$. We next present a three-step algorithm that leverages the associativity of matrix multiplication to avoid materializing the $h$ heads of NoPE keys and values, thereby optimizing the decoding efficiency.
\paragraph{Step 1 (Query-Side Weight Absorption).}
We first reorganize the NoPE key and value up-projection matrices into head-wise tensors:
\begin{equation*}
\tilde{\tens{W}}^{\text{UK}}_{i,:,:} := \left(\bm{W}^{\text{UK}}_{:,(i)}\right)^{\top}\in\mathbb{R}^{d_h\times d_c},
\qquad
\tilde{\tens{W}}^{\text{UV}}_{i,:,:} := \bm{W}^{\text{UV}}_{:,(i)}\in\mathbb{R}^{d_c\times d_h},
\end{equation*}
where $\tilde{\tens{W}}^{\text{UK}}\in \mathbb{R}^{h \times d_h \times d_c}$ and $\tilde{\tens{W}}^{\text{UV}}\in \mathbb{R}^{h \times d_c \times d_h}$. For the NoPE query at position $n-1$, $\tens{Q}_{n-1, :, :}^{\text{NoPE}}$, we absorb the up-projection weight tensor $\tilde{\tens{W}}^{\text{UK}}$ directly into the query via Einstein summation:
\begin{equation*}
\tilde{\tens{Q}}_{n-1, :, :}^{\text{NoPE}} = \text{einsum}\left(\texttt{"hp,hpc->hc"}, \, \tens{Q}_{n-1, :, :}^{\text{NoPE}}, \,  \tilde{\tens{W}}^{\text{UK}} \right), \quad p=d_h, \, c=d_c, \quad \tilde{\tens{Q}}_{n-1, :, :}^{\text{NoPE}} \in \mathbb{R}^{h\times d_c}.
\end{equation*}
\paragraph{Step 2 (MQA-Style Decoding on Latent KV Cache).}
Given the KV cache $\bm{C}^{\text{KV}}$ and $\bm{K}^{\text{RoPE}}$, we define the shared key and value tensors by concatenating and reshaping the latent representations as $\tilde{\tens{K}}= \operatorname{Reshape}\left(\operatorname{Concat}\left(\left[\bm{C}^{\text{KV}}, \, \bm{K}^{\text{RoPE}}\right], \, \text{dim=1}\right), \, \left[n, \, 1, \, d_c + d_h^R\right] \right) \in \mathbb{R}^{n \times 1 \times (d_c + d_h^R)}$ and $\tilde{\tens{V}} = \operatorname{Reshape}\left(\bm{C}^{\text{KV}}, \, \left[n, \, 1, \, d_c\right] \right) \in \mathbb{R}^{n \times 1 \times d_c}$. Under this formulation, decoding reduces to an MQA-style attention mechanism in which the attention logits (i.e., query--key inner products before softmax) are computed in a $(d_c+d_h^R)$-dimensional space using these shared KV states. Incorporating the concatenated query $\tilde{\tens{Q}}_{n-1, :, :}=\operatorname{Concat}\left(\left[\tilde{\tens{Q}}_{n-1, :, :}^{\text{NoPE}}, \, \tens{Q}_{n-1, :, :}^{\text{RoPE}}\right], \, \text{dim=1}\right) \in \mathbb{R}^{h\times (d_c+d_h^R)}$, the attention output is calculated as follows:
\begin{equation}
\label{eq:step-2 decoding}
\small
\tens{Z}_{n-1, :, :}=\operatorname{Attention}\left(\tilde{\tens{Q}}_{n-1, :, :}, \, \operatorname{RepeatInterleave}\left(\tilde{\tens{K}}, \, h, \, \text{dim}=1\right), \,  \operatorname{RepeatInterleave}\left(\tilde{\tens{V}}, \, h, \, \text{dim}=1\right)\right),
\end{equation}
where $\tens{Z}_{n-1, :, :}\in\mathbb{R}^{h \times d_c}$. FlashAttention-3~\citep{shah2024flashattention} and FlashMLA~\citep{flashmla2025} provide highly optimized kernels designed to implement the Step-2 decoding computation directly.
\paragraph{Step 3 (Output Up-Projection).}
Finally, the up-projection tensor maps the intermediate attention output to the final attention output:
\begin{equation*}
\tens{O}_{n-1, :, :} = \text{einsum}\left(\texttt{"hc,hcp->hp"}, \, \tens{Z}_{n-1, :, :}, \, \tilde{\tens{W}}^{\text{UV}} \right), \quad  c=d_c, \, p=d_h, \quad \tens{O}_{n-1, :, :} \in \mathbb{R}^{h\times d_h}.
\end{equation*}
\paragraph{Block Multiplications.}
\label{Sec:mla_block}
For each head $i$, we define the constituent sub-blocks $\bm{W}_{(b),(i)}^{\cdot} \in \mathbb{R}^{d_h \times d_h}$ by partitioning the up-projection matrices into $d_h$-sized row blocks for $b \in \{0, 1, 2, 3\}$:
\begin{equation*}
\begin{split}
    \bm{W}_{(b),(i)}^{\text{UK}} & := \bm{W}^{\text{UK}}\left[b d_h : (b+1) d_h, \, i d_h : (i+1) d_h\right], \\
    \bm{W}_{(b),(i)}^{\text{UV}} & := \bm{W}^{\text{UV}}\left[b d_h : (b+1)d_h, \, i d_h : (i+1) d_h\right].
\end{split}
\end{equation*}
Consequently, each head-specific up-projection matrix can be expressed as a vertical stack of these four row-blocks:
\begin{equation*}
\bm{W}_{:, (i)}^{\text{UK}} =
\begin{bmatrix}
\bm{W}_{(0),(i)}^{\text{UK}}\\
\bm{W}_{(1),(i)}^{\text{UK}}\\
\bm{W}_{(2),(i)}^{\text{UK}}\\
\bm{W}_{(3),(i)}^{\text{UK}}
\end{bmatrix}, \quad
\bm{W}_{:, (i)}^{\text{UV}} =
\begin{bmatrix}
\bm{W}_{(0),(i)}^{\text{UV}}\\
\bm{W}_{(1),(i)}^{\text{UV}}\\
\bm{W}_{(2),(i)}^{\text{UV}}\\
\bm{W}_{(3),(i)}^{\text{UV}}
\end{bmatrix}.
\end{equation*}
Similarly, we partition the KV latent matrix $\bm{C}^{\text{KV}} \in \mathbb{R}^{n \times d_c}$ into horizontal channel blocks $\bm{C}_{:, (b)}^{\text{KV}} := \bm{C}^{\text{KV}}\left[:, \, b d_h : (b+1)d_h\right]$, such that $\bm{C}^{\text{KV}} = \left[ \bm{C}_{:, (0)}^{\text{KV}}, \dots, \bm{C}_{:, (3)}^{\text{KV}} \right]$. This block decomposition allows the key and value projections for head $i$ to be reformulated as a sum of four sub-block products:
\begin{equation}
\label{eq:mla_block_recon}
\tens{K}_{:, (i), :}^{\text{NoPE}} = \sum_{b=0}^{3}\bm{C}_{:, (b)}^{\text{KV}}\bm{W}_{(b),(i)}^{\text{UK}}, \qquad \tens{V}_{:, (i), :} = \sum_{b=0}^{3}\bm{C}_{:, (b)}^{\text{KV}}\bm{W}_{(b),(i)}^{\text{UV}}.
\end{equation}
\subsection{Grouped Latent Attention}
\label{Sec:GLA-2}
Grouped Latent Attention (GLA-2)~\citep{zadouri2025hardware} bisects MLA's single latent head into two latent heads, using the first latent head $(\bm{C}_{:, (0)}^{\text{KV}},\bm{C}_{:, (1)}^{\text{KV}})$ for the first half of attention heads and the second latent head $(\bm{C}_{:, (2)}^{\text{KV}},\bm{C}_{:, (3)}^{\text{KV}})$ for the second half. We define the group-mapping function as:
\begin{equation}
\label{eq:gla-2_define_group}
\gamma(i)=
\begin{cases}
0, & i < h/2,\\
1, & i \ge h/2,
\end{cases}
\qquad
\bar{i}= i-\frac{\gamma(i)\,h}{2} \in \{0,\dots,h/2-1\}.
\end{equation}
Let $\bm{W}^{(\gamma(i)),\text{UK}},\bm{W}^{(\gamma(i)),\text{UV}}\in\mathbb{R}^{2d_h\times (h/2)\,d_h}$ denote the up-projection matrices for latent group $\gamma(i)\in\{0,1\}$. We extract the head-specific slices for head $i$ by indexing into these matrices:
\begin{equation*}
\bm{W}_{:, (i)}^{(\gamma(i)),\text{UK}}=\bm{W}^{(\gamma(i)),\text{UK}}\left[:, \, \bar{i} d_h : (\bar{i}+1) d_h\right],\ \
\bm{W}_{:,(i)}^{(\gamma(i)),\text{UV}}=\bm{W}^{(\gamma(i)),\text{UV}}\left[:, \, \bar{i} d_h : (\bar{i}+1) d_h\right],
\end{equation*}
where $\bm{W}_{:, (i)}^{(\gamma(i)),\text{UK}},\bm{W}_{:, (i)}^{(\gamma(i)),\text{UV}}\in\mathbb{R}^{2d_h\times d_h}$. To further facilitate block-wise computation, we partition these slices into $d_h$-row blocks $\bm{W}_{(b),(i)}^{(\gamma(i)),\cdot}$ for $b \in \{0, 1\}$, defined as:
\begin{equation*}
\begin{split}
    \bm{W}_{(b),(i)}^{(\gamma(i)),\text{UK}} & := \bm{W}^{(\gamma(i)),\text{UK}}\left[b d_h : (b+1) d_h, \, \bar{i} d_h : (\bar{i}+1) d_h\right], \\
    \bm{W}_{(b),(i)}^{(\gamma(i)),\text{UV}} & := \bm{W}^{(\gamma(i)),\text{UV}}\left[bd_h : (b+1) d_h, \, \bar{i} d_h : (\bar{i}+1) d_h\right],
\end{split}
\end{equation*}
where each block $\bm{W}_{(b),(i)}^{(\gamma(i)),\text{UK}}, \bm{W}_{(b),(i)}^{(\gamma(i)),\text{UV}} \in \mathbb{R}^{d_h \times d_h}$. This partitioning allows us to decompose the head-specific up-projection matrices into two row-wise blocks:
\begin{equation*}
\bm{W}_{:, (i)}^{(\gamma(i)),\text{UK}} =
\begin{bmatrix}
\bm{W}_{(0),(i)}^{(\gamma(i)),\text{UK}}\\
\bm{W}_{(1),(i)}^{(\gamma(i)),\text{UK}}
\end{bmatrix}, \qquad
\bm{W}_{:, (i)}^{(\gamma(i)),\text{UV}} =
\begin{bmatrix}
\bm{W}_{(0),(i)}^{(\gamma(i)),\text{UV}}\\
\bm{W}_{(1),(i)}^{(\gamma(i)),\text{UV}}
\end{bmatrix}.
\end{equation*}
Consequently, the NoPE key and value computations for head $i$ can be expressed as the summation of two block products:
\begin{equation}
\label{eq:gla2_block_recon}
\begin{split}
\tens{K}_{:, (i), :}^{\text{NoPE}} &= \bm{C}_{:, (2\gamma(i))}^{\text{KV}} \bm{W}_{(0),(i)}^{(\gamma(i)),\text{UK}} + \bm{C}_{:, (2\gamma(i)+1)}^{\text{KV}} \bm{W}_{(1),(i)}^{(\gamma(i)),\text{UK}}, \\
\tens{V}_{:, (i), :} &= \bm{C}_{:, (2\gamma(i))}^{\text{KV}} \bm{W}_{(0),(i)}^{(\gamma(i)),\text{UV}} + \bm{C}_{:, (2\gamma(i)+1)}^{\text{KV}} \bm{W}_{(1),(i)}^{(\gamma(i)),\text{UV}}.
\end{split}
\end{equation}

\section{Multi-Head Low-Rank Attention}
Building on the block decompositions in Sections~\ref{Sec:mla_block} and \ref{Sec:GLA-2}, we propose MLRA. By shifting the summation from KV computation to attention output, MLRA treats each block projection as an independent low-rank branch and sums their outputs. MLRA is illustrated in Figures~\ref{fig:mlra_2} and~\ref{fig:mlra_4}.
\subsection{MLRA-4}
By substituting the block-partitioned identities from Eq.~(\ref{eq:mla_block_recon}) into the attention mechanism, the output for head $i$ can be expressed as:
\begin{equation*}
\begin{split}
\tens{O}_{:, i, :}
=&\ \operatorname{Softmax}\left(
\tau\tens{Q}_{:, i, :}^{\text{NoPE}}\left(\sum_{b=0}^{3}\bm{C}_{:, (b)}^{\text{KV}}\bm{W}_{(b),(i)}^{\text{UK}}\right)^{\top}
+\tau\tens{Q}_{:, i, :}^{\text{RoPE}}\left(\bm{K}^{\text{RoPE}}\right)^{\top}
\right)\left(\sum_{b=0}^{3}\bm{C}_{:, (b)}^{\text{KV}}\bm{W}_{(b),(i)}^{\text{UV}}\right).
\end{split}
\end{equation*}
Motivated by Eq.~(\ref{eq:mla_block_recon}), we propose MLRA-4, which computes attention independently on each blockwise branch and sums the resulting outputs:
\begin{equation}
\label{eq:mlra4}
\tens{O}_{:, i, :}
=
\sum_{b=0}^{3}
\operatorname{Softmax}\left(
\tau\tens{Q}_{:, i, :}^{\text{NoPE}}\left(\bm{C}_{:, (b)}^{\text{KV}}\bm{W}_{(b),(i)}^{\text{UK}}\right)^{\top}
+\tau\tens{Q}_{:, i, :}^{\text{RoPE}}\left(\bm{K}^{\text{RoPE}}\right)^{\top}
\right)\left(\bm{C}_{:, (b)}^{\text{KV}}\bm{W}_{(b),(i)}^{\text{UV}}\right).
\end{equation}
\subsection{MLRA-2}
\label{Sec:MLRA-2}
Following the grouping logic of GLA-2 from Eq.~(\ref{eq:gla-2_define_group}) and substituting the block-wise identities from Eq.~(\ref{eq:gla2_block_recon}), the attention output for GLA-2 can be expressed as:
\begin{equation*}
\small
\begin{split}
\tens{O}_{:, i, :}
=&\ \operatorname{Softmax}\left(
\tau\tens{Q}_{:, i, :}^{\text{NoPE}}
\left(\bm{C}_{:, (2\gamma(i))}^{\text{KV}}\bm{W}_{(0),(i)}^{(\gamma(i)),\text{UK}}+\bm{C}_{:, (2\gamma(i)+1)}^{\text{KV}}\bm{W}_{(1),(i)}^{(\gamma(i)),\text{UK}}\right)^{\top}
+\tau\tens{Q}_{:, i, :}^{\text{RoPE}}\left(\bm{K}^{\text{RoPE}}\right)^{\top}
\right)\\
&\ \cdot
\left(\bm{C}_{:, (2\gamma(i))}^{\text{KV}}\bm{W}_{(0),(i)}^{(\gamma(i)),\text{UV}}+\bm{C}_{:, (2\gamma(i)+1)}^{\text{KV}}\bm{W}_{(1),(i)}^{(\gamma(i)),\text{UV}}\right).
\end{split}
\end{equation*}
Analogously, we derive MLRA-2 by moving the block summation outside the attention operator, yielding a sum of two branchwise attention outputs:
\begin{equation}
\label{eq:mlra2}
\small
\begin{split}
\tens{O}_{:, i, :}
&=\ \operatorname{Softmax}\left(
\tau\tens{Q}_{:, i, :}^{\text{NoPE}}\left(\bm{C}_{:, (2\gamma(i))}^{\text{KV}}\bm{W}_{(0),(i)}^{(\gamma(i)),\text{UK}}\right)^{\top}
+\tau\tens{Q}_{:, i, :}^{\text{RoPE}}\left(\bm{K}^{\text{RoPE}}\right)^{\top}
\right)\left(\bm{C}_{:, (2\gamma(i))}^{\text{KV}}\bm{W}_{(0),(i)}^{(\gamma(i)),\text{UV}}\right)\\
&+\operatorname{Softmax}\left(
\tau\tens{Q}_{:, i, :}^{\text{NoPE}}\left(\bm{C}_{:,(2\gamma(i)+1)}^{\text{KV}}\bm{W}_{(1),(i)}^{(\gamma(i)),\text{UK}}\right)^{\top}
+\tau\tens{Q}_{:, i, :}^{\text{RoPE}}\left(\bm{K}^{\text{RoPE}}\right)^{\top}
\right)\left(\bm{C}_{:, (2\gamma(i)+1)}^{\text{KV}}\bm{W}_{(1),(i)}^{(\gamma(i)),\text{UV}}\right).
\end{split}
\end{equation}
MLRA-2 and MLRA-4 differ primarily in their latent-to-head mapping and branching factor. In MLRA-2, each latent block is up-projected to serve $h/2$ heads, with the final output resulting from a two-branch summation. Conversely, MLRA-4 utilizes four latent blocks that are each up-projected to serve all $h$ heads, resulting in a four-branch summation. Despite these structural differences, both variants decompose the computation into independent branches that require only a final reduction. This architecture naturally facilitates 4-way TP decoding, reducing the per-head attention logit space to $1.5 d_h$ after absorption—a significant reduction compared to $4.5 d_h$ in MLA and $2.5 d_h$ in GLA-2.
\subsection{Scaling Query/Key--Value Latent States and Attention Output}
\label{sec:scale}
Recent work~\citep{team2025longcat} observes that the RoPE key $\left(\bm{K}^{\text{RoPE}}\right)$ can exhibit a significant variance mismatch relative to other attention components $\left(\tens{Q}, \, \tens{K}^{\text{NoPE}}, \, \tens{V}\right)$. This discrepancy arises in MLA because RMSNorm~\citep{zhang2019root} is applied to the latent states $\bm{H}\bm{W}^{\text{DQ}}$ and $\bm{H}\bm{W}^{\text{DKV}}$ prior to the up-projections that generate the final query, NoPE key, and value tensors. To formally investigate this variance instability, we introduce the following assumption:
\begin{assumption}
\label{assump:stats}
We assume that all elements of the weight matrices $\bm{W}^{\text{DQ}}$, $\bm{W}^{\text{DKV}}$, $\bm{W}^{\text{UQ}}$, $\bm{W}^{\text{QR}}$, $\bm{W}^{\text{UK}}$, $\bm{W}^{\text{KR}}$, and $\bm{W}^{\text{UV}}$ are independent and identically distributed (i.i.d.) random variables with zero mean and variance $\sigma_{w}^{2}$. Furthermore, these weight matrices are assumed to be mutually independent of the input signals at each layer. Finally, for the multi-branch of MLRA, we assume the attention outputs $\tens{O}_{(b), (i), :}$ originating from different latent blocks $b$ are mutually uncorrelated, implying that $\operatorname{Cov}\left(\tens{O}_{(a), (i), :}, \tens{O}_{(b), (i),  :}\right) \approx 0$ for all $a \neq b$.
\end{assumption}

\paragraph{RoPE Key Variance.}
Recall that MLA computes the RoPE key as $\bm{K}^{\text{RoPE}}=\operatorname{RoPE}\left(\bm{H}\bm{W}^{\text{KR}}\right)$. Let $\overline{\bm{K}}^{\text{RoPE}}_{t,u}:=(\bm{H}\bm{W}^{\text{KR}})_{t,u} =\sum_{m=1}^{d}\bm{H}_{t,m}\bm{W}^{\text{KR}}_{m,u}$. Since the hidden states $\bm{H}$ are RMS-normalized, $\mathbb{E}\!\left[(\bm{H}_{t,m})^2\right]\approx 1$.
With $\mathbb{E}\!\left[\bm{W}^{\text{KR}}_{m,u}\right]=0$ and
$\mathbb{E}\!\left[\left(\bm{W}^{\text{KR}}_{m,u}\right)^2\right]=\sigma_{w}^2$, we have
\begin{equation*}
\operatorname{Var}\!\left(\overline{\bm{K}}^{\text{RoPE}}_{t,u}\right)
=\sum_{m=1}^{d}\Big(
\mathbb{E}\!\left[(\bm{H}_{t,m})^2\right]\,
\mathbb{E}\!\left[(\bm{W}^{\text{KR}}_{m,u})^2\right]
-\big(\mathbb{E}[\bm{H}_{t,m}]\,\mathbb{E}[\bm{W}^{\text{KR}}_{m,u}]\big)^2
\Big)
\approx \sum_{m=1}^{d} 1\cdot \sigma_{w}^2
=d\sigma_{w}^2.
\end{equation*}
Since RoPE is an orthogonal transformation, it does not change the variance:
\begin{equation}
\operatorname{Var}\!\left(\bm{K}^{\text{RoPE}}\right)\approx d\sigma_{w}^2.
\label{eq:var_k_rope_short}
\end{equation}
\paragraph{NoPE Key Variance.}
Next, we derive the variance of the NoPE keys, $\overline{\bm{K}}^{\text{NoPE}} = \bm{C}^{\text{KV}}\bm{W}^{\text{UK}}$. By definition, the latent KV state $\bm{C}^{\text{KV}} = \operatorname{RMSNorm}(\bm{H}\bm{W}^{\text{DKV}})$ is constrained such that each element has approximately unit mean square, i.e., $\mathbb{E}\left[(\bm{C}^{\text{KV}}_{t, l})^2\right] \approx 1$. Considering an arbitrary entry $\overline{\bm{K}}^{\text{NoPE}}_{t, u} = \sum_{l=1}^{d_c} \bm{C}^{\text{KV}}_{t,l}\bm{W}^{\text{UK}}_{l,u}$, and assuming the weights $\bm{W}^{\text{UK}}$ are zero-mean with variance $\sigma_{w}^2$, the variance of the product is:
\begin{equation}
\operatorname{Var}\!\left(\overline{\bm{K}}^{\text{NoPE}}_{t,u}\right)
=\sum_{l=1}^{d_c}\Big(
\mathbb{E}\!\left[(\bm{C}^{\text{KV}}_{t,l})^2\right]\,
\mathbb{E}\!\left[(\bm{W}^{\text{UK}}_{l,u})^2\right]
-\big(\mathbb{E}[\bm{C}^{\text{KV}}_{t,l}]\,\mathbb{E}[\bm{W}^{\text{UK}}_{l,u}]\big)^2
\Big)
\approx \sum_{l=1}^{d_c} 1\cdot \sigma_{w}^2
=d_c\sigma_{w}^2.
\label{eq:var_k_nope_short}
\end{equation}
Because reshaping does not alter the underlying variance, $\operatorname{Var}\!\left(\tens{K}^{\text{NoPE}}\right)$ remains $d_c\sigma_{w}^2$. Extending this derivation to the remaining attention components, we obtain the variances for the value and query tensors as follows:
\begin{equation*}
\operatorname{Var}\!\left(\tens{V}\right)\approx d_c\sigma_{w}^2,
\qquad
\operatorname{Var}\!\left(\tens{Q}^{\text{NoPE}}\right)\approx d_c^\prime\sigma_{w}^2,
\qquad
\operatorname{Var}\!\left(\tens{Q}^{\text{RoPE}}\right)\approx d_c^\prime\sigma_{w}^2.
\end{equation*}
\paragraph{Variance Mismatch and Calibration.}
Comparing our variance derivations shows that $\frac{\operatorname{Var}\left(\bm{K}^{\text{RoPE}}\right)}{\operatorname{Var}\left(\tens{K}^{\text{NoPE}}\right)} \approx \frac{d}{d_c}$, which explains the mismatch noted in~\citet{team2025longcat} when the latent dimension $d_c$ is much smaller than $d$. This is corrected by applying scaling factors to the latent states before the up-projection. Specifically, using $\alpha_{q}=\sqrt{d/d_c^\prime}$ and $\alpha_{kv}=\sqrt{d/d_c}$ ensures that the query and NoPE key components $\left(\tens{Q}, \, \tens{K}^{\text{NoPE}}\right)$ achieve parity with the variance of the RoPE key.

Adopting the variance-calibration strategy from \citet{team2025longcat}, we apply analogous rescaling to our MLRA variants. For MLRA-2 and MLRA-4, we rescale the query and KV latent states to ensure that the variance of NoPE queries and keys aligns with that of the partial RoPE components across all branches.
\begin{equation}
\label{eq:mlra4_rescale}
\bm{C}^{\text{Q}} \leftarrow \sqrt{\frac{d}{d_c^\prime}}\,\bm{C}^{\text{Q}}, 
\qquad 
\bm{C}^{\text{KV}} \leftarrow \sqrt{\frac{4d}{d_c}}\,\bm{C}^{\text{KV}}.
\end{equation}
Since summing the attention outputs from multiple branches alters the variance, we apply a rescaling factor to the attention outputs of MLRA-2 and MLRA-4 as follows:
\begin{equation}
\text{(MLRA-2)}\quad \tens{O}_{:, i, :} \leftarrow \frac{1}{\sqrt{2}}\,\tens{O}_{:, i, :},
\qquad
\text{(MLRA-4)}\quad \tens{O}_{:, i, :} \leftarrow \frac{1}{2}\,\tens{O}_{:, i, :}.
\end{equation}
\begin{remark}
Although these weight matrices are typically initialized with zero mean and a common variance, these conditions are not guaranteed during training. Consequently, Assumption~\ref{assump:stats} may not strictly hold in practice. However, the effectiveness of this scaling is best assessed through ablation studies, the results of which are detailed in Section~\ref{sec:ex_scale}.
\end{remark}
\begin{table}[t]
\caption{Comparison of parameters and KV cache loading among attention mechanisms. Some results are taken from~\citet{zhang2025tensor} and~\citet{zadouri2025hardware}. For attention mechanism details, refer to Appendix~\ref{appendix:attention_mechanism}.}
\vspace{-0.15in}
\begin{center}
\resizebox{\linewidth}{!}{
\begin{tabular}{ccccccc}
\toprule
\textbf{Method} & \textbf{\# Parameters} & \textbf{KV Cache} &
\begin{tabular}{@{}c@{}}\textbf{Loading Per Token}\\\textbf{Per Device (1 GPU)}\end{tabular} &
\begin{tabular}{@{}c@{}}\textbf{Loading Per Token}\\\textbf{Per Device (2 GPUs)}\end{tabular} &
\begin{tabular}{@{}c@{}}\textbf{Loading Per Token}\\\textbf{Per Device (4 GPUs)}\end{tabular} &
\begin{tabular}{@{}c@{}}\textbf{Loading Per Token}\\\textbf{Per Device (8 GPUs)}\end{tabular} \\
\midrule
MHA~\citep{vaswani2017attention}
    & $4dhd_h$
    & $2hd_{h}$
    & $128d_{h}$ & $64d_{h}$ & $32d_{h}$ & $16d_{h}$ \\
MQA~\citep{shazeer2019fast}
    & $2 d d_h\left(h+1\right)$
    & $2d_h$
    & $2d_h$ & $2d_h$ & $2d_h$ & $2d_h$ \\
GQA~\citep{ainslie2023gqa}
    & $2 d d_h \left(h+g\right)$
    & $2gd_h$
    & $16d_h$ & $8d_h$ & $4d_h$ & $2d_h$ \\
MLA~\citep{liu2024deepseek}
    & \begin{tabular}{@{}c@{}}$d_c^{\prime}\left(d+h d_h+h d_h^R\right)+d d_h^R$ \\ $+\, d_c\left(d+2 h d_h\right)+dhd_h$\end{tabular}
    & $d_c+d_h^R$
    & $4.5d_h$ & $4.5d_h$ & $4.5d_h$ & $4.5d_h$ \\
MFA~\citep{hu2024multi}
    & \begin{tabular}{@{}c@{}}$d_c^{\prime}\left(d+h\cdot 2d_h\right)$ \\ $+2d\cdot2d_h+dh\cdot2d_h$\end{tabular}
    & $4d_h$
    & $4d_h$ & $4d_h$ & $4d_h$ & $4d_h$ \\
TPA~\citep{zhang2025tensor}
    & $d(\beta_{q}+2\beta_{kv})\left(h+d_h\right)+dhd_h$
    & $2\beta_{kv}\left(h+d_h\right)$
    & $6d_h$ & $5d_h$ & $4.5d_h$ & $4.25d_h$ \\
GLA-2~\citep{zadouri2025hardware}
    & \begin{tabular}{@{}c@{}}$d_c^{\prime}\left(d+h d_h+h d_h^R\right)+d d_h^R$ \\ $+\, d_c\left(d+h d_h\right)+dhd_h$\end{tabular}
    & $d_c+d_h^R$
    & $4.5d_h$ & $2.5d_h$ & $2.5d_h$ & $2.5d_h$ \\
GTA~\citep{zadouri2025hardware}
    & $dhd_h+dgd_h+dd_h^R+dhd_h$
    & $gd_h+d_h^R$
    & $8.5d_h$ & $4.5d_h$ & $2.5d_h$ & $1.5d_h$ \\
\midrule
MLRA-2
    & \begin{tabular}{@{}c@{}}$d_c^{\prime}\left(d+h d_h+h d_h^R\right)+d d_h^R$ \\ $+\, d_c\left(d+h d_h\right)+dhd_h$\end{tabular}
    & $d_c+d_h^R$
    & $4.5d_h$ & $2.5d_h$ & $1.5d_h$ & $1.5d_h$ \\
MLRA-4
    & \begin{tabular}{@{}c@{}}$d_c^{\prime}\left(d+h d_h+h d_h^R\right)+d d_h^R$ \\ $+\, d_c\left(d+2 h d_h\right)+dhd_h$\end{tabular}
    & $d_c+d_h^R$
    & $4.5d_h$ & $2.5d_h$ & $1.5d_h$ & $1.5d_h$ \\
\bottomrule
\end{tabular}
\label{tab:attention_compare}
}
\end{center}
\end{table}
\subsection{Analysis} 
\paragraph{KV Cache.}
We evaluate the per-device KV cache loading under various TP configurations using Qwen3-32B~\citep{qwen3} and Kimi-K2~\citep{Bai2025KimiKO} as base architectures. Qwen3-32B utilizes GQA with 64 query heads and 8 KV heads ($d_h=128$), setting $g=8$ KV heads. Kimi-K2 adopts MLA with 64 heads and a partial RoPE dimension ($d_h^{R}=64$). For TPA, we maintain the original configuration with $\beta_{kv}=2$. Table~\ref{tab:attention_compare} summarizes the per-device KV cache loading under TP as the number of devices increases. To support TP, the official MLA decoding implementation, FlashMLA~\citep{flashmla2025}, distributes up-projection matrices across devices by head. However, this approach leads to redundant KV cache loading; as a result, the per-device loading remains constant at $4.5d_h$ regardless of the TP degree. TPA constructs its $h$ key-value heads as linear combinations of $\beta_{kv}$ shared heads. It supports TP only for the combination coefficients, while the shared heads must be redundantly loaded by each device. Consequently, the per-device KV cache loading is $4d_h + \frac{2d_h}{\varphi}$, where $\varphi$ denotes the number of TP devices. GLA-2 partially addresses this by partitioning the latent head into two smaller latent heads, reducing the per-device loading to $2.5d_h$ under 2-way TP. Notably, for MLA with TP $> 1$ and GLA-2 with TP $> 2$, the KV cache loading becomes invariant to the number of devices due to sharding constraints, causing the per-device loading to plateau at $4.5 d_h$ and $2.5 d_h$, respectively. While GQA and GTA require 8-way TP to reduce the per-device loading to $2d_h$ and $1.5d_h$, MLRA achieves $1.5d_h$ with only 4-way TP.
\paragraph{Attention Decoding Arithmetic Intensity.}
Arithmetic intensity (AI)~\citep{williams2009roofline}, defined as the ratio of floating-point operations to memory access (FLOPs/byte), serves as a critical metric for identifying whether a workload is memory-bound or compute-bound~\citep{zadouri2025hardware}. Given that the context length $n$ is the dominant factor in long-context decoding, we evaluate the arithmetic intensity (AI) of various attention mechanisms, with the results summarized in Table~\ref{tab:arithmetic-intensity}. MLRA-2 and MLRA-4 achieve AI values of $h$ and $2h$, respectively, maintaining the high arithmetic intensity characteristic of MLA and GLA-2. By significantly increasing the compute-to-memory ratio, MLRA shifts the decoding process away from the HBM bandwidth ceiling toward a compute-limited regime.
\begin{table}[t]
\caption{Comparison of attention decoding arithmetic intensity among attention mechanisms.}
\vspace{-0.1in}
\resizebox{\linewidth}{!}{
\centering
\setlength{\tabcolsep}{3pt}
\begin{tabular}{ccccccccccc}
\toprule
\textbf{Method} & \textbf{MHA} & \textbf{MQA}  & \textbf{GQA} & \textbf{MLA} & \textbf{MFA} & \textbf{TPA} & \textbf{GLA-2} & \textbf{GTA} &   \textbf{MLRA-2} &  \textbf{MLRA-4} \\
\midrule
\multirow{2}{*}{\parbox{2cm}{\centering \textbf{Arithmetic} \\ \textbf{Intensity}}} &
$\frac{4nhd_h}{4nhd_h}$ & 
$\frac{4nhd_h}{4nd_h}$ & 
$\frac{4nhd_h}{4ngd_h}$ & 
$\frac{4nhd_c+2nh d_h^R}{2 n\left(d_c+d_h^R\right)}$ & 
$\frac{4nh\cdot 2d_h}{4n\cdot2d_h}$ & 
$\frac{4 nh \beta_{kv} d_h+4 nh d_h}{4n\beta_{kv}\left(h+d_h\right)}$ &
$\frac{2nh \frac{d_c}{2}+nh d_h^R}{2 n\left(\frac{d_c}{2}+d_h^R\right)}$ &
$\frac{4 n h d_h}{2 n \left(g d_h+d_h^R\right)}$ & 
$\frac{2 nh \frac{d_c}{4}+nhd_h^R}{2 n\left(\frac{d_c}{4}+d_h^R\right)}$ &
$\frac{4 nh \frac{d_c}{4}+2 nh  d_h^R}{2 n\left(\frac{d_c}{4}+d_h^R\right)}$ \\[3pt]
& $\approx 1$ & $\approx h$  & $\approx \frac{h}{g}$ & $\approx 2h$ & $\approx h$ & $\approx \frac{\left(1+\beta_{kv}\right)hd_h}{\beta_{kv}\left(h+d_h\right)}$ & $\approx h$ & $\approx \frac{2h}{g}$ & $\approx h$&$\approx 2h$ \\[4pt]
\bottomrule
\end{tabular}
}
\label{tab:arithmetic-intensity}
\end{table}

\section{Experiments}
\subsection{Experimental Setup}
\paragraph{Model Configuration.} 
We adopt the Llama-3~\citep{Dubey2024TheL3} architecture (Appendix~\ref{appendix:llama3}) and compare MLRA against the following attention mechanism baselines: MHA~\citep{vaswani2017attention}, MQA~\citep{shazeer2019fast}, GQA~\citep{ainslie2023gqa}, MLA~\citep{liu2024deepseek}, MFA~\citep{hu2024multi}, TPA~\citep{zhang2025tensor}, GLA-2~\citep{zadouri2025hardware}, GLA-4, and GTA~\citep{zadouri2025hardware}. GLA-4 compresses the KV cache into four latent heads. We initialize the MHA baseline with the Llama3.2-3B~\citep{Dubey2024TheL3} configuration and use it as our parameter-count reference. Following \citet{zadouri2025hardware}, for each other attention variant, we adjust the Feed-Forward Network (FFN) intermediate dimension to match the total number of parameters of this MHA baseline. Full details of the architectural hyperparameters are provided in Appendix~\ref{appendix:main_parameters}. All models are implemented on top of the \texttt{nanoGPT}~\citep{NanoGPT} codebase.
\paragraph{Pretraining Configuration.} 
We pretrain all models at the 2.9B-parameter scale on FineWeb-Edu-100B~\citep{penedo2024the}. Each model is pretrained from scratch on 98.3B tokens, with an additional 0.1B tokens for validation. We use the GPT-2 tokenizer with a vocabulary size of 50,304 and follow the standard GPT-3 pretraining setup. We use AdamW~\citep{loshchilov2019decoupled} as the optimizer with $\left(\beta_1, \beta_2\right)=(0.9,0.95), \epsilon=10^{-8}$, weight decay 0.1, and gradient clipping at 1.0. The learning rate is linearly warmed up for the first 2,000 steps, then annealed with cosine decay~\citep{loshchilov2017sgdr} to $10 \%$ of the peak. Peak learning rate is $1.6 \times 10^{-4}$. We train with a context length of 2,048 tokens and a global batch size of 480 sequences (983,040 tokens per step, $\approx$1.0M) for 100,000 steps. All models are pretrained on 8 NVIDIA H100 80GB GPUs.
\paragraph{Evaluation Benchmark.} 
In addition to evaluating the perplexity from the FineWeb-Edu validation dataset, we evaluate our models on six additional datasets: Wikipedia, C4~\citep{raffel2020exploring}, the Pile~\citep{gao2020pile}, RefinedWeb~\citep{penedo2023the}, Cosmopedia~\citep{benallal2024cosmopedia}, and FineWeb~\citep{penedo2024the} using 0.1B tokens per dataset. We evaluate zero-shot performance on common-sense reasoning benchmarks, including ARC-Easy (ARC-E)~\citep{yadav-etal-2019-quick}, ARC-Challenge (ARC-C), OpenBookQA~\citep{mihaylov-etal-2018-suit}, BoolQ~\citep{clark-etal-2019-boolq}, HellaSwag~\citep{zellers2019hellaswag}, Winogrande~\citep{sakaguchi2021winogrande}, and PIQA~\citep{bisk2020piqa}, using the \texttt{lm-evaluation-harness}~\citep{eval-harness} package. We report normalized accuracy for ARC-E/C, OpenBookQA, HellaSwag, and PIQA, with standard accuracy for all other tasks.
\subsection{Preliminary Ablation Results}
\begin{figure}[t]
    \centering
    \begin{minipage}{0.479\textwidth}
        \centering
        \includegraphics[width=\linewidth]{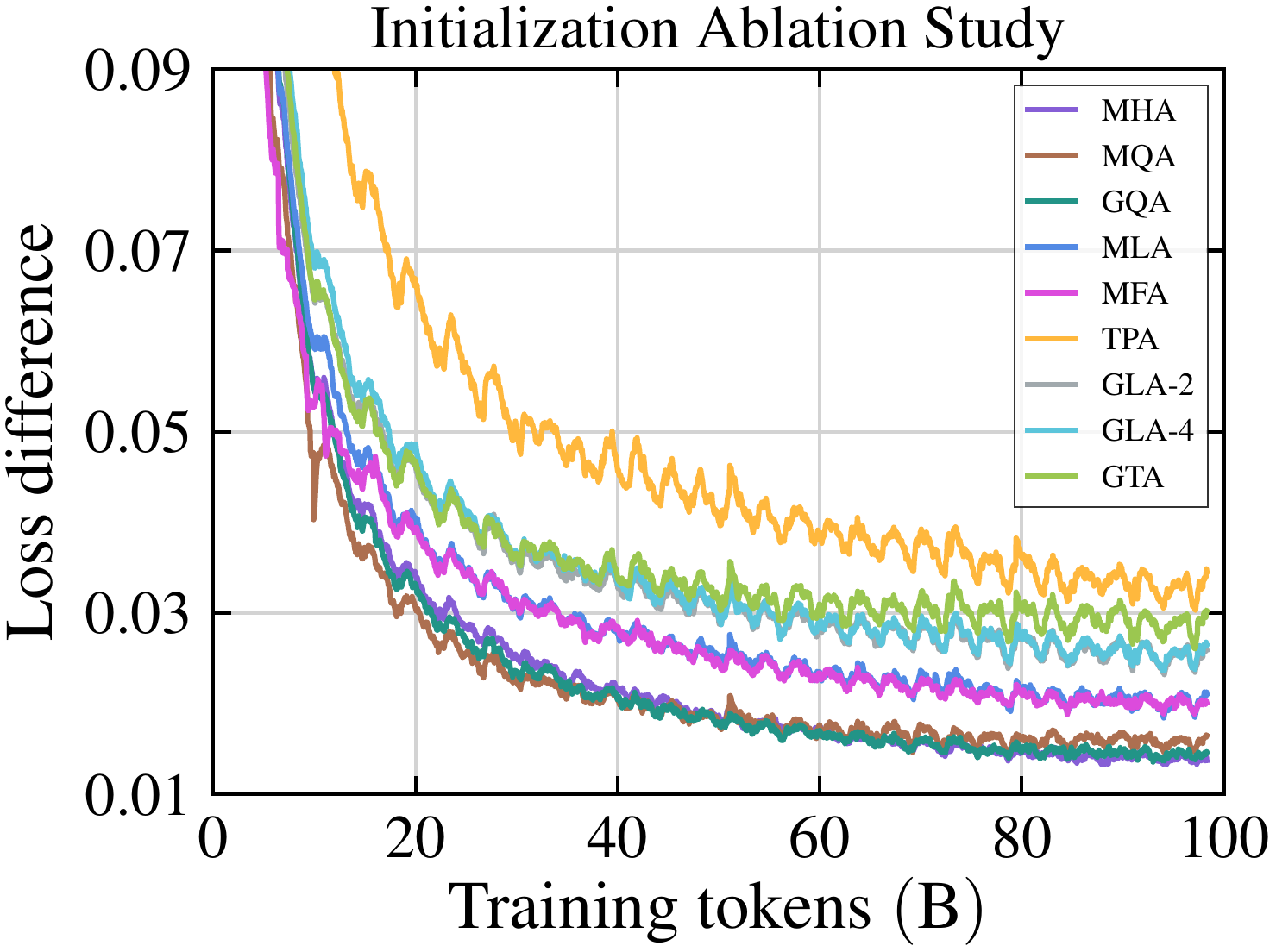}
        \caption{Loss difference between $\mathcal{N}(0, \sigma=0.02)$ and zero initialization, calculated by subtracting the loss of the latter from the former.}
        \label{fig:ablation_initialization}
    \end{minipage}%
    \hfill
    \begin{minipage}{0.479\textwidth}
        \centering
        \includegraphics[width=\linewidth]{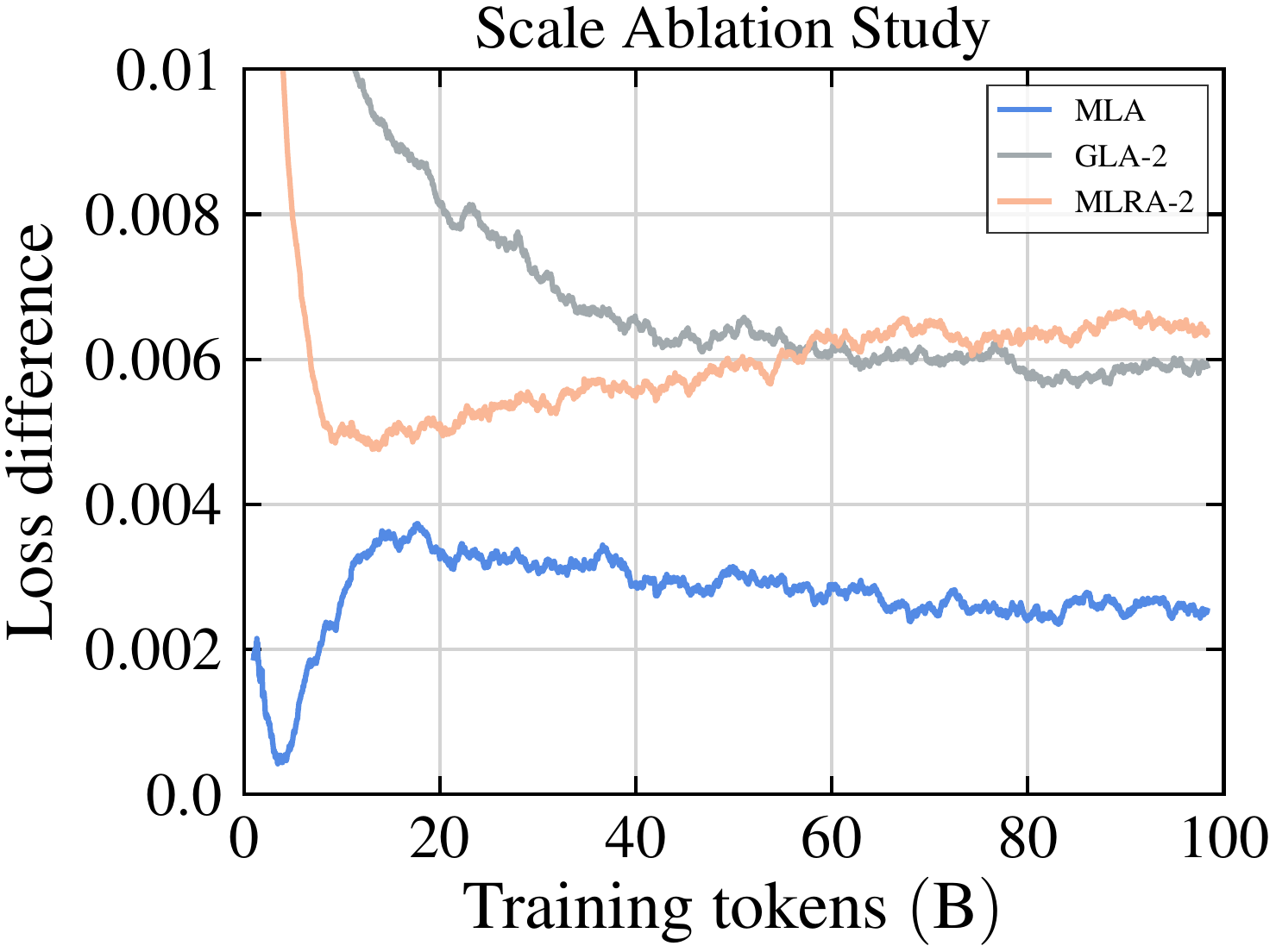}
        \caption{Loss difference between models without and with scaling, calculated by subtracting the loss of the latter from the former.}
        \label{fig:ablation_scale}
    \end{minipage}

    \vspace{1em}

    \begin{minipage}{0.479\textwidth}
        \centering
        \includegraphics[width=\linewidth]{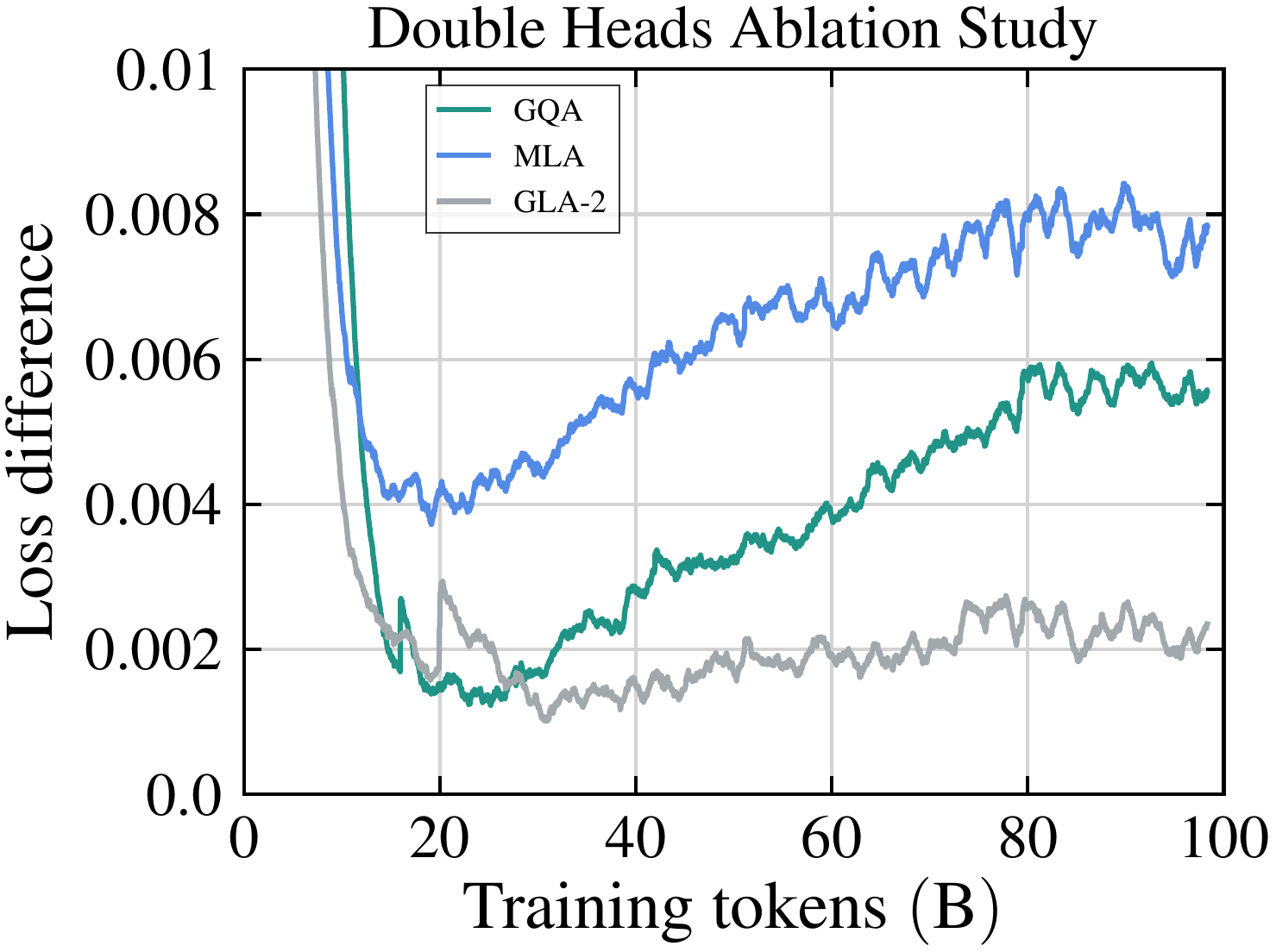}
        \caption{Loss difference between models with and without double heads, calculated by subtracting the loss of the latter from the former.}
        \label{fig:ablation_double_heads}
    \end{minipage}%
    \hfill
    \begin{minipage}{0.479\textwidth}
        \centering
        \includegraphics[width=\linewidth]{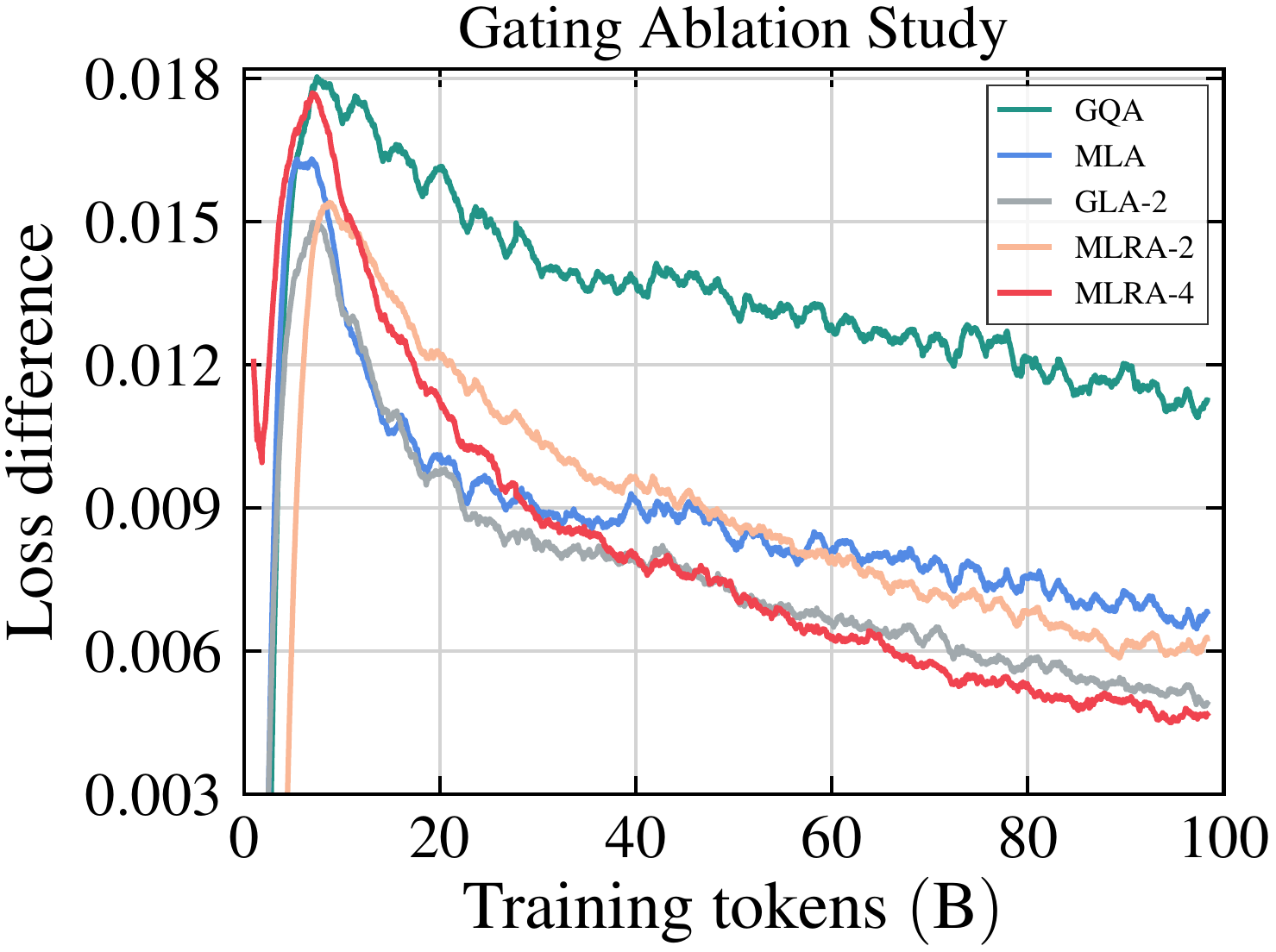}
        \caption{Loss difference between models without and with gating, calculated by subtracting the loss of the latter from the former.}
        \label{fig:ablation_gated}
    \end{minipage}
\end{figure}
\subsubsection{Initialization}
We follow the GPT~\citep{radford2018improving} initialization method, where all model weights are initialized using an $\mathcal{N}(0, \sigma=0.02)$ distribution. However, TPA employs zero initialization for the output projection parameters of the attention and FFN modules, which is an approach also explored in muP~\citep{yang2021tuning} and LoRA~\citep{hu2022lora}. To evaluate these different approaches, we conduct an ablation study comparing zero initialization against $\mathcal{N}(0, \sigma=0.02)$ for the output projection parameters across all models. It is important to note that for MLA, GLA-2, and GLA-4, we apply scaling as discussed in Section~\ref{sec:scale}. As illustrated in Figure~\ref{fig:ablation_initialization} and Table~\ref{tab:ablation_initialization}, the results for loss and perplexity demonstrate that zero initialization outperforms the $\mathcal{N}(0, \sigma=0.02)$ distribution. Unless otherwise specified, all models in the following experiments utilize this zero initialization.
\subsubsection{Scaling}
\label{sec:ex_scale}
We evaluate the effectiveness of the scaling on MLA, GLA-2, and MLRA-2. As illustrated in Figure~\ref{fig:ablation_scale}, all three models exhibit improved convergence when scaling is applied. As shown in Table~\ref{tab:ablation_scale}, all three models achieve lower average perplexity after scaling. Notably, MLA and GLA-2 show more substantial improvements, while MLRA-2 yields a marginal gain. Unless otherwise specified, MLA, GLA, and MLRA in the following experiments utilize this scaling.
\subsubsection{Double Heads}
While MLRA-2 and MLRA-4 do not increase the number of query heads, their multi-branch design increases the number of attention heads involved in computation. Consequently, we double the number of attention heads for GQA, MLA, and GLA-2 while keeping the KV-cache size fixed to evaluate whether this increase contributes to performance gains. To maintain a constant parameter budget during this adjustment, we reduce the FFN intermediate sizes; the corresponding architectural hyperparameters are detailed in Appendix~\ref{appendix:double_heads_parameters}. As illustrated by the loss curves in Figure~\ref{fig:ablation_double_heads} and the results in Table~\ref{tab:ablation_double_heads}, doubling the number of attention heads leads to higher loss and fails to decrease perplexity across all three models. These findings suggest that doubling the number of attention heads does not yield any measurable performance improvement. Unless otherwise specified, GQA, MLA, and GLA use the default head count (no head doubling).
\begin{table}[t]
\caption{Validation perplexity (lower is better) across seven datasets: Wikipedia, C4, Pile, RefinedWeb, Cosmopedia, FineWeb, and FineWeb-Edu. The best results are indicated in \textbf{bold}, while the second best are \underline{underlined}.}
\vspace{-0.1in}
\resizebox{\linewidth}{!}{
\centering
\setlength{\tabcolsep}{5pt}
\begin{tabular}{c | c c c c c c c | c}
\toprule
\textbf{Method} & \textbf{Wikipedia} & \textbf{C4} & \textbf{Pile} & \textbf{RefinedWeb} & \textbf{Cosmopedia} & \textbf{FineWeb} & \textbf{FineWeb-Edu} & \textbf{Avg} \\
\midrule
MHA    & 14.624          & 16.575          & \textbf{12.929} & 18.698          & 9.102           & 15.656          & 9.434           & 13.860          \\
MQA    & 15.134          & 16.837          & 14.008          & 19.202          & 9.484           & 15.942          & 9.533           & 14.306          \\
GQA    & 15.057          & 16.628          & 13.758          & 18.885          & 9.504           & 15.713          & 9.427           & 14.139          \\
MLA    & 14.567 & 16.345       & \underline{12.965} & 18.523       & \underline{8.966} & 15.440 & 9.284 & \underline{13.727} \\
MFA    & 15.693          & 16.738          & 13.903          & 19.125          & 9.423           & 15.815          & 9.506           & 14.315          \\
TPA    & 14.789          & 16.622          & 13.333          & 18.971          & 9.130           & 15.717          & 9.333           & 13.985          \\
GLA-2  & 14.605          & \underline{16.323} & 13.225       & \underline{18.509} & 9.118          & \underline{15.424}          & 9.249           & 13.779          \\
GLA-4  & \underline{14.547}          & 16.436          & 13.229          & 18.578          & 9.076           & 15.535          & 9.307           & 13.815          \\
GTA    & 14.733          & 16.599          & 13.402          & 18.924          & 9.129           & 15.672          & 9.346           & 13.972          \\
\midrule
MLRA-2 & 14.615          & 16.342          & 13.236          & 18.602          & 9.153           & 15.439          & \underline{9.242}           & 13.804          \\
MLRA-4 & \textbf{14.407} & \textbf{16.286} & 13.124          & \textbf{18.398} & \textbf{8.937}  & \textbf{15.361} & \textbf{9.193}  & \textbf{13.672} \\
\bottomrule
\end{tabular}
}
\label{tab:ppl}
\end{table}
\begin{table}[t]
\caption{Downstream evaluation on seven common-sense reasoning benchmarks: ARC-E, ARC-C, OpenBookQA, BoolQ, HellaSwag, Winogrande, and PIQA. ARC-E/C, OpenBookQA, HellaSwag, and PIQA use normalized accuracy (\%); others use standard accuracy (\%). Best is \textbf{bold}; second best is \underline{underlined}.}
\vspace{-0.1in}
\resizebox{\linewidth}{!}{
\centering
\setlength{\tabcolsep}{8pt}
\begin{tabular}{c | c c c c c c c | c}
\toprule
\textbf{Method} & \textbf{ARC-E} & \textbf{ARC-C} & \textbf{OpenBookQA} & \textbf{BoolQ} & \textbf{HellaSwag} & \textbf{Winogrande} & \textbf{PIQA} & \textbf{Avg} \\
\toprule
MHA    & \underline{69.11} & 39.16          & 40.80          & 62.26          & 60.82          & 57.62          & 74.86          & 57.81          \\
MQA    & 66.16          & 38.31          & 41.80          & 62.05          & 60.24          & 59.83          & 74.48          & 57.55          \\
GQA    & 67.13          & 39.42          & 42.00          & 63.39          & 61.29          & 56.91          & 75.08          & 57.89          \\
MLA    & 68.22          & 39.16          & \underline{42.60} & \textbf{64.10} & 61.39          & \underline{60.06} & \textbf{75.68} & \underline{58.75} \\
MFA    & 69.02          & 39.93          & 42.40          & 63.49          & 60.72          & 58.96          & 75.19          & 58.53          \\
TPA    & \textbf{69.44} & 40.61          & 41.60          & 60.03          & 61.02          & 57.85          & 74.54          & 57.87          \\
GLA-2  & 68.01          & 40.19          & 40.60          & \underline{63.94} & 61.54 & 58.41          & 75.41 & 58.30          \\
GLA-4  & 68.77          & 41.04          & 41.20          & 61.96          & \underline{61.61}          & 58.09          & 74.65          & 58.19          \\
GTA    & 67.97          & 39.68          & \underline{42.60} & 59.72          & 61.03          & 58.48          & 75.14          & 57.80          \\
\midrule
MLRA-2 & 67.89          & \textbf{42.24} & 42.00          & 61.65          & 61.49          & 59.98          & \underline{75.52}         & 58.68          \\
MLRA-4 & 67.63          & \underline{41.38} & \textbf{43.00} & 61.74          & \textbf{62.16} & \textbf{61.48} & 74.48          & \textbf{58.84} \\
\bottomrule
\end{tabular}
}
\label{tab:common_sense_reasoning}
\end{table} 
\begin{table}[t]
\caption{Validation perplexity (lower is better) w/ gating across seven datasets. The best results are indicated in \textbf{bold}, while the second best are \underline{underlined}.}
\vspace{-0.1in}
\resizebox{\linewidth}{!}{
\centering
\setlength{\tabcolsep}{5pt}
\begin{tabular}{c | c c c c c c c | c}
\toprule
\textbf{Method} & \textbf{Wikipedia} & \textbf{C4} & \textbf{Pile} & \textbf{RefinedWeb} & \textbf{Cosmopedia} & \textbf{FineWeb} & \textbf{FineWeb-Edu} & \textbf{Avg} \\
\midrule
GQA w/ gating    & \underline{14.362} & 16.484          & 13.113          & 18.696          & 9.098           & 15.581          & 9.311           & 13.806          \\
MLA w/ gating   & \textbf{14.346}    & 16.297          & \textbf{12.866} & 18.456          & 8.936 & 15.383        & 9.212 & \underline{13.642} \\
GLA-2 w/ gating  & 14.597             & 16.286 & \underline{12.997} & 18.473       & 8.986           & 15.369 & 9.198        & 13.701          \\
\midrule
MLRA-2 w/ gating & 14.424             & \underline{16.252} & 13.017          & \underline{18.407} & \underline{8.924} & \underline{15.351} & \underline{9.180} & 13.651 \\
MLRA-4 w/ gating & 14.431             & \textbf{16.170} & 13.073          & \textbf{18.386} & \textbf{8.874}  & \textbf{15.266} & \textbf{9.148}  & \textbf{13.621} \\
\bottomrule
\end{tabular}
}
\label{tab:ppl_gated}
\end{table}
\subsection{Main Results}
As shown in Table~\ref{tab:ppl}, MLRA-4 achieves the best average perplexity (13.672), outperforming all other models, including MLA (13.727). Notably, MLRA-4 also delivers the lowest perplexity on FineWeb-Edu (9.193). Furthermore, Table~\ref{tab:common_sense_reasoning} demonstrates that MLRA-4 attains the highest average zero-shot accuracy across all common-sense reasoning tasks. This consistent superiority of MLRA-4 over MLRA-2 across both evaluations highlights the benefits of increasing the number of branches.
\subsection{Gated Attention}
Following \citet{qiu2025gated}, we introduce a gating mechanism prior to the attention output projection (Appendix~\ref{appendix:gated_attention}). To maintain a constant parameter budget, we reduce the FFN intermediate size accordingly; detailed architectural hyperparameters are provided in Appendix~\ref{appendix:gated_parameters}. As illustrated in Figure~\ref{fig:ablation_gated}, all five models exhibit improved convergence with gating applied. As shown in Table~\ref{tab:ppl_gated}, gating consistently improves perplexity across all models, with MLRA-4 achieving the best overall average perplexity and MLRA-2 attaining performance comparable to MLA (13.651 vs. 13.642).
\begin{figure}[t]
    \centering
    \begin{minipage}{0.479\textwidth}
        \centering
        \includegraphics[width=\linewidth]{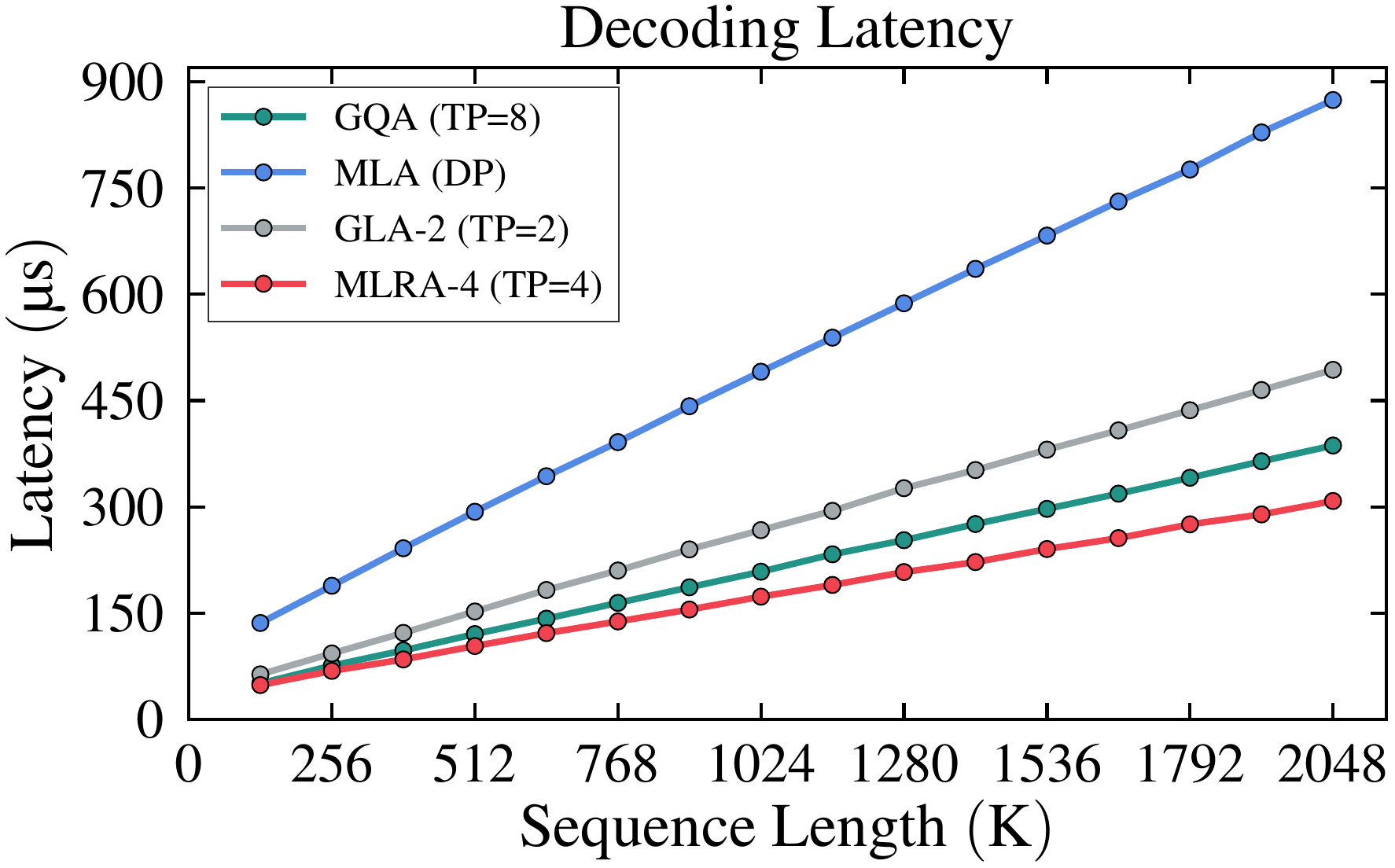}
        \caption{Decoding latency (lower is better) versus sequence length (batch=1) for GQA, MLA, GLA-2, and MLRA-4.}
        \label{fig:decoding_speed}
    \end{minipage}%
    \hfill
    \begin{minipage}{0.479\textwidth}
        \centering
        \includegraphics[width=\linewidth]{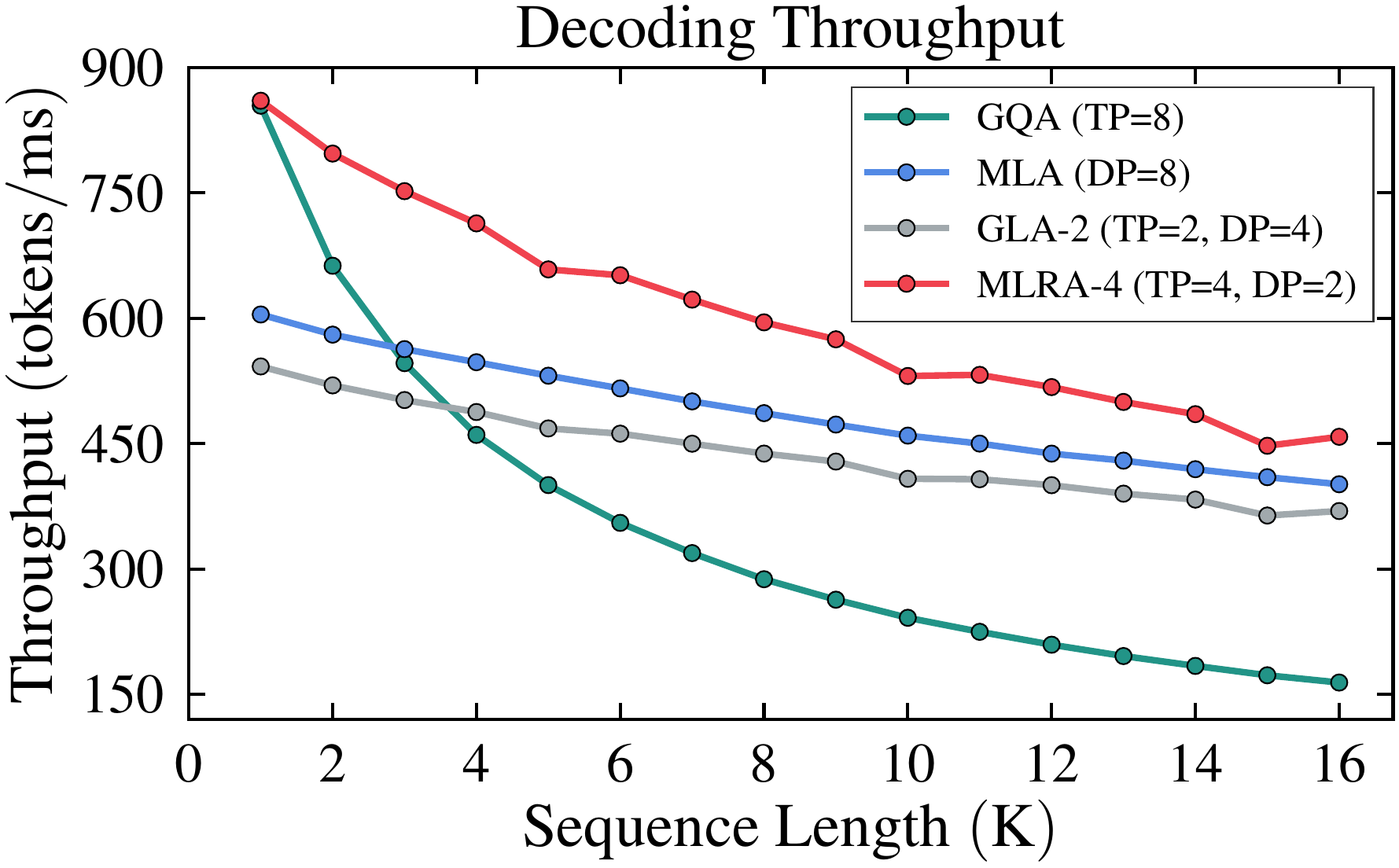}
        \caption{Decoding throughput versus sequence length (batch=128) for GQA, MLA, GLA-2, and MLRA-4.}
        \label{fig:decoding_throughput}
    \end{minipage}
\end{figure}
\subsection{Decoding Efficiency}
\paragraph{Decoding Speed.}
We benchmark long-context attention decoding speed for GQA, MLA, GLA-2, and MLRA-4 on an NVIDIA H100 80GB GPU. For MLA, GLA-2, and MLRA-4, we follow the attention decoding formulation in Eq.~(\ref{eq:step-2 decoding}). All models use 64 heads with a head dimension of 128; for MLA, GLA-2, and MLRA-4, the partial RoPE dimension is 64. MLA is evaluated using DeepSeek’s official implementation FlashMLA~\citep{flashmla2025}. GQA and GLA-2 use FlashAttention-3 kernels~\citep{dao2022flashattention,dao2023flashattention2,shah2024flashattention}. We implement our MLRA-4 kernel based on FlashAttention-3. We evaluate decoding speed across sequence lengths from 131,072 to 2,097,152 tokens (128K--2M). As shown in Figure~\ref{fig:decoding_speed}, MLRA-4 consistently outperforms all baselines at every length, yielding 1.05$\times$--1.26$\times$ speedups over GQA. The gap grows with context length against GQA and GLA-2, while the speedup over MLA remains steady at about 2.8$\times$, indicating that MLRA-4 with TP=4 substantially reduces long-context decoding latency.

\paragraph{Decoding Throughput.}
We evaluate decoding throughput for GQA, MLA, GLA-2, and MLRA-4 on eight NVIDIA H100 80GB GPUs, fixing the number of attention heads to 128 and the hidden size to 7168, following DeepSeekV3~\citep{DeepSeekAI2024DeepSeekV3TR}. We set $g=16$ for GQA. For MLA decoding deployment, there is a trade-off between data parallelism (DP) and tensor parallelism. With DP, we assign different requests to different devices, so attention parameters are replicated across devices and the load can become imbalanced due to varying sequence lengths. With TP, the up-projection parameters are sharded by head, but the KV cache loading is repeated across devices. Following SGLang~\citep{zheng2024sglang}, we aim to eliminate redundant KV cache loading. Therefore, we use DP=8 for MLA, TP=2/DP=4 for GLA-2, TP=4/DP=2 for MLRA-4, and TP=8 for GQA. Throughput is reported for sequence lengths ranging from 1,024 to 16,384 tokens, and our end-to-end measurements include both the pre-attention stage that prepares inputs for the attention kernel and the attention computation itself. We accelerate pre-attention computation with torch.compile~\citep{paszke2019pytorch} for MLA, GLA-2, and MLRA-4, and with custom Triton kernels for GQA. As shown in Figure~\ref{fig:decoding_throughput}, MLRA-4 achieves the highest decoding throughput across both short and long sequence lengths. This suggests that MLRA-4 with TP=4/DP=2 reduces parameter redundancy relative to MLA’s DP=8, while introducing only modest partial RoPE duplication, thereby yielding higher throughput than MLA. For short sequences, GQA outperforms MLA and GLA-2 because pre-attention dominates latency. However, MLRA-4 remains competitive with GQA due to having even fewer query, key, and value parameters, as shown in Appendix~\ref{appendix:main_parameters}.
\section{Conclusion}
We propose Multi-Head Low-Rank Attention (MLRA), a novel attention mechanism with native 4-way tensor parallelism support. At the 2.9B scale, MLRA-4 achieves state-of-the-art performance on perplexity and zero-shot common-sense reasoning benchmarks. Furthermore, MLRA achieves the lowest decoding latency for long-context sequences (up to 2M tokens) and the highest throughput across sequence lengths from 1K to 16K tokens with 4-way tensor parallelism.

\newpage
\section*{Acknowledgement}
We thank Songlin Yang for helpful discussion. We thank all the anonymous reviewers for their helpful comments and suggestions.
\bibliography{iclr2026_conference}
\bibliographystyle{iclr2026_conference}
\newpage
\appendix
\onecolumn
\renewcommand{\appendixpagename}{\centering \LARGE Appendix}
\appendixpage

\startcontents[section]
\printcontents[section]{l}{1}{\setcounter{tocdepth}{2}}
\vspace{20ex}

\newpage

\section{Notation}
\label{appendix:notation}
\begin{table*}[htbp]
\centering
\small
\renewcommand{\arraystretch}{1.15}
\caption{Notation used throughout this paper.}
\label{tab:notation}
\begin{tabularx}{\textwidth}{@{} c c X @{}}
\toprule
\textbf{Symbol} & \textbf{Shape / Type} & \textbf{Meaning} \\
\midrule
$n$ & scalar & Sequence length (number of tokens). \\
$d$ & scalar & Model/hidden dimension. \\
$h$ & scalar & Number of attention heads. \\
$d_h$ & scalar & Per-head dimension. \\
$d_h^R$ & scalar & Partial RoPE dimension. \\
$d_r$ & scalar & RoPE rotation dimension in Theorem~\ref{theorem:rope} (even); typically $d_r=d_h$ or $d_r=d_h^R$. \\
$d_f$ & scalar & MLP intermediate (FFN) dimension. \\
$g$ & scalar & Number of groups (or KV heads in MQA/GQA). \\
$r$ & scalar & Repeat factor $r=h/g$ when $g$ KV heads are broadcast to $h$ query heads. \\
$\alpha_q,\alpha_{kv},\alpha_{attn}$ & scalar & Variance-calibrating rescaling factors for query/KV latents and attention outputs. \\
$\beta_q,\beta_{kv}$ & scalar & Number of low-rank components in TPA for queries / keys-values. \\
$d_c$ & scalar & Latent KV dimension in MLA/GLA. \\
$d_c^\prime$ & scalar & Latent Q dimension in MLA/GLA/MFA. \\
\midrule
$s$ & integer & Translation offset. \\
$t_q,t_k$ & integer & Query/key token positions. \\
$b$ & integer & Block index. \\
$i$ & integer & Head index ($i\in\{0, \, \ldots, \, h-1\}$). \\
$j$ & integer & Group index ($j\in\{0, \, \ldots, \, g-1\}$). \\
$\gamma(i)$ & integer & Mapping from head index to group index.\\
$\bar{i}$ & integer & Head index within group in GLA-2, $\bar{i}= i-\gamma(i)h/2$. \\
$\tau$ & scalar & Softmax scaling factor, $\tau=1/\sqrt{d_h+d_h^R}$. \\
$\varphi$ & integer & Number of tensor-parallel devices. \\
\midrule
$\operatorname{RoPE}\left(\cdot\right)$ & function & Rotary Position Embedding applied to vectors (implemented via rotation matrices). \\
$\operatorname{RMSNorm}\left(\cdot\right)$ & function & Root-mean-square normalization. \\
$\operatorname{Reshape}\left(\cdot\right)$ & operator & Tensor reshaping (no data change). \\
$\operatorname{RepeatInterleave}\left(\cdot\right)$ & operator & Replication along the head dimension (e.g., broadcasting $g$ KV heads to $h$ heads). \\
$\operatorname{Concat}\left(\cdot\right)$ & operator & Concatenation along the last dimension unless stated otherwise. \\
\bottomrule
\end{tabularx}
\end{table*}

\section{Theorem}
\label{appendix:theorem1}

\subsection{Translation Equivariance}
Equivariance is a fundamental property in geometric systems such as molecules, where vector features such as atomic forces or dipole moments must transform consistently with the coordinate system~\citep{weiler20183d,fuchs2020se,satorras2021n}. In the context of sequence models, a common transformation is sequence translation. Let
$\bm{X} = \left(\bm{x}^{(0)}, \bm{x}^{(1)}, \ldots, \bm{x}^{(n-1)}\right) \in \mathbb{X}$
be a sequence of tokens, and define a translation operator $T_s: \mathbb{X} \rightarrow \mathbb{X}$ that translates the entire sequence by $s$ positions. Let $\phi: \mathbb{X} \rightarrow \mathbb{Y}$ be a function that maps a sequence to a matrix of attention scores $\phi\left(\bm{X}\right) \in \mathbb{R}^{n \times n}$, where each element
$\phi\left(\bm{X}\right)_{t_q,t_k} = A\left(\bm{x}^{(t_q)}, \bm{x}^{(t_k)}\right)$
denotes the attention score between tokens $\bm{x}^{(t_q)}$ and $\bm{x}^{(t_k)}$. We say that $\phi$ is \textit{translation equivariant} if there exists a corresponding output-space transformation $S_s: \mathbb{Y} \rightarrow \mathbb{Y}$ such that:
\begin{equation}
\label{eq:equivariance}
\phi\left(T_s\left(\bm{X}\right)\right) = S_s\left(\phi\left(\bm{X}\right)\right), \quad \forall s.
\end{equation}
This property ensures that the attention score between two tokens depends only on their relative positions, not their absolute positions. That is crucial for batch inference using left padding, where sequences of different lengths are offset to align ends. The first non-padding token of a sequence is no longer at position $0$, yet attention scores remain equivariant to this translation.

\subsection{Rotary Position Embedding}
Rotary Position Embedding (RoPE)~\citep{su2024roformer} is a positional encoding method designed to incorporate relative position information directly into the attention mechanism. We show in this section that RoPE is translation equivariant.

\begin{theorem}
\label{theorem:rope}
Given two tokens with query $\bm{q}$ and key $\bm{k}$ at positions $t_q$ and $t_k$, respectively, let $\operatorname{RoPE}\left(\bm{q}, t_q\right)$ and $\operatorname{RoPE}\left(\bm{k}, t_k\right)$ denote the RoPE-encoded vectors. We show that translating both positions by an offset $s$ leaves the inner product unchanged:
\begin{equation*}
\left\langle\operatorname{RoPE}\left(\bm{q}, t_q+s\right), \operatorname{RoPE}\left(\bm{k}, t_k+s\right)\right\rangle=
\left\langle\operatorname{RoPE}\left(\bm{q}, t_q\right), \operatorname{RoPE}\left(\bm{k}, t_k\right)\right\rangle.
\end{equation*}
Equivalently, for the attention-score matrix $\phi(\bm{X})\in\mathbb{R}^{n\times n}$ induced by RoPE-based dot products, $\phi$ is translation equivariant in the sense of Eq.~(\ref{eq:equivariance}) with $S_s$ being the simultaneous row/column shift operator.
\end{theorem}

\begin{proof}
Let the RoPE dimension be $d_r$ (assumed even). RoPE applies a block-diagonal rotation matrix $\bm{R}_t\in\mathbb{R}^{d_r\times d_r}$ at position $t$. Writing RoPE as right-multiplication on row vectors, we compute the inner product under RoPE as:
\begin{equation*}
\left(\bm{q}\bm{R}_{t_q} \right)\left( \bm{k}\bm{R}_{t_k}\right)^{\top}
=\bm{q} \bm{R}_{t_q} \bm{R}_{t_k}^\top \bm{k}^\top
= \operatorname{Re}\left[\sum_{\ell=0}^{d_r/2-1}
\bm{q}_{\left[2\ell:2\ell+2\right]}\, \bm{k}_{\left[2\ell:2\ell+2\right]}^* \, e^{\mathrm{i}\left(t_q-t_k\right)\theta_\ell}\right],
\end{equation*}
where $\theta_\ell$ is the angular frequency for the $\ell$-th 2D block and $(\cdot)^*$ denotes complex conjugation under the standard $\mathbb{R}^2\simeq\mathbb{C}$ identification.

Now consider translating both tokens by an offset $s$. The relative displacement is unchanged:
\begin{equation*}
\left(t_q + s\right) - \left(t_k + s\right) = t_q - t_k,
\end{equation*}
so the factor $e^{\mathrm{i}\left(t_q-t_k\right)\theta_\ell}$ remains unchanged for every $\ell$. Therefore,
\begin{equation*}
\left\langle\operatorname{RoPE}\left(\bm{q}, t_q+s\right), \operatorname{RoPE}\left(\bm{k}, t_k+s\right)\right\rangle=
\left\langle\operatorname{RoPE}\left(\bm{q}, t_q\right), \operatorname{RoPE}\left(\bm{k}, t_k\right)\right\rangle.
\end{equation*}
This proves RoPE-induced dot-product scores satisfy translation equivariance in Eq.~(\ref{eq:equivariance}).
\end{proof}

\begin{remark}
\label{remark:rope_equivariance}
While RoPE preserves dot-product translation equivariance, applying an \emph{arbitrary} linear map after RoPE generally breaks this property. Specifically, consider:
\begin{equation*}
\left\langle \operatorname{RoPE}\left(\bm{q}, t_q\right)\bm{W}^{\text{Q}} ,\  \operatorname{RoPE}\left(\bm{k}, t_k\right)\bm{W}^{\text{K}} \right\rangle 
= \bm{q}\bm{R}_{t_q} \bm{W}^{\text{Q}} \left(\bm{W}^{\text{K}}\right)^{\top}\bm{R}_{t_k}^{\top} \bm{k}^{\top}.
\end{equation*}
The term $\bm{W}^{\text{Q}} \left(\bm{W}^{\text{K}}\right)^{\top}$ breaks translation equivariance by disrupting the expression's dependence on the relative position $t_q-t_k$. The property would only be preserved in the specific case where $\bm{W}^{\text{Q}} \left(\bm{W}^{\text{K}}\right)^{\top}= \bm{I}$, which would reduce the expression to its original form. However, since this constraint is difficult to enforce during training, translation equivariance is generally lost when applying a linear projection after RoPE.
\end{remark}

\section{Attention Mechanism}
\label{appendix:attention_mechanism}

\subsection{Multi-Head Attention (MHA)}
Consider a sequence of $n$ tokens with hidden states $\bm{H} \in \mathbb{R}^{n \times d}$. We first project these hidden states into queries, keys, and values using projection matrices
$\bm{W}^{\text{Q}}, \bm{W}^{\text{K}}, \bm{W}^{\text{V}} \in \mathbb{R}^{d \times \left(hd_h\right)}$:
\begin{equation*}
\overline{\bm{Q}}= \operatorname{RoPE} \left(\bm{H} \bm{W}^{\text{Q}}\right),\quad 
\overline{\bm{K}} = \operatorname{RoPE} \left(\bm{H} \bm{W}^{\text{K}}\right),\quad 
\overline{\bm{V}} = \bm{H} \bm{W}^{\text{V}},
\end{equation*}
where $\overline{\bm{Q}}, \overline{\bm{K}}, \overline{\bm{V}} \in \mathbb{R}^{n \times \left(hd_h\right)}$, $h$ is the number of attention heads, and $d_h$ is the dimensionality of each head.
Next, we reshape these matrices to separate the head dimension:
\begin{equation*}
\small
\tens{Q} = \operatorname{Reshape}\left(\overline{\bm{Q}}, \, \left[n, \, h, \, d_h\right]\right),\quad
\tens{K}^{\text{C}} = \operatorname{Reshape}\left(\overline{\bm{K}}, \, \left[n, \, h, \, d_h\right]\right),\quad 
\tens{V}^{\text{C}} = \operatorname{Reshape}\left(\overline{\bm{V}}, \, \left[n, \, h, \, d_h\right]\right),
\end{equation*}
such that $\tens{Q}, \tens{K}^{\text{C}}, \tens{V}^{\text{C}} \in \mathbb{R}^{n \times h \times d_h}$. We cache $\tens{K}^{\text{C}}$ and $\tens{V}^{\text{C}}$ to accelerate inference.

\begin{remark}
Let $\tens{Q}, \tens{K}^{\text{C}} \in \mathbb{R}^{n \times h \times d_h}$ denote the RoPE-encoded queries and keys after projection and reshaping. For head $i$, define:
\begin{equation*}
\tens{Q}_{t_q,i,:} := \tens{Q}\left[t_q, \, i, \, :\right], \quad 
\tens{K}_{t_k,i,:} := \tens{K}^{\text{C}}\left[t_k, \, i, \, :\right].
\end{equation*}
Then, for any translation offset $s$, it follows from Theorem~\ref{theorem:rope} that:
\begin{equation*}
\left\langle \tens{Q}_{t_q+s,i,:}, \, \tens{K}_{t_k+s,i,:} \right\rangle 
= \left\langle \tens{Q}_{t_q,i,:}, \, \tens{K}_{t_k,i,:} \right\rangle.
\end{equation*}
\end{remark}

\subsection{Multi-Query Attention (MQA) and Grouped-Query Attention (GQA)}
Both MQA and GQA reduce the number of key and value heads compared to MHA, while maintaining the full number of query heads. MQA takes this to the extreme by using a single key-value head for all query heads, whereas GQA partitions the query heads into groups that each share a key-value head. Given a sequence of $n$ tokens with hidden states $\bm{H} \in \mathbb{R}^{n \times d}$, the queries are computed using the same projection as in MHA. To reduce KV cache, we use projection matrices $\bm{W}^{\text{K}}, \bm{W}^{\text{V}} \in \mathbb{R}^{d \times d_h g}$, where $g < h$ (e.g., $g = 1$ for MQA), to compute:
\begin{equation*}
\overline{\bm{K}}^{\text{C}} = \operatorname{RoPE}\left(\bm{H} \bm{W}^{\text{K}}\right), \quad 
\overline{\bm{V}}^{\text{C}} = \bm{H} \bm{W}^{\text{V}}.
\end{equation*}
These are then reshaped into per-head form:
\begin{equation*}
\tens{K}^{\text{C}} = \operatorname{Reshape}\left(\overline{\bm{K}}^{\text{C}}, \, \left[n, \, g, \, d_h\right]\right), \quad 
\tens{V}^{\text{C}} = \operatorname{Reshape}\left(\overline{\bm{V}}^{\text{C}}, \, \left[n, \, g, \, d_h\right]\right).
\end{equation*}
We cache $\tens{K}^{\text{C}}$ and $\tens{V}^{\text{C}}$ during inference. We repeat them by a factor of $r=h/g$ along the head axis to match the $h$ query heads:
\begin{equation*}
\begin{split}
\tens{K}&=\operatorname{RepeatInterleave}\left(\tens{K}^{\text{C}}, \, r, \, \text{dim}=1\right) \in \mathbb{R}^{n \times h \times d_h},\\
\tens{V}&=\operatorname{RepeatInterleave}\left(\tens{V}^{\text{C}}, \, r, \, \text{dim}=1\right) \in \mathbb{R}^{n \times h \times d_h}.
\end{split}
\end{equation*}

\begin{remark}
Let $\tens{Q} \in \mathbb{R}^{n \times h \times d_h}$ and $\tens{K}^{\text{C}} \in \mathbb{R}^{n \times g \times d_h}$ denote the RoPE-encoded queries and cached keys, respectively. For head $i$, define:
\begin{equation*}
\tens{Q}_{t_q,i,:} := \tens{Q}\left[t_q, \, i, \, :\right], \quad
\tens{K}_{t_k,i,:} := \tens{K}^{\text{C}}\left[t_k, \, \left\lfloor\frac{i}{r}\right\rfloor, \, :\right].
\end{equation*}
Since both vectors are RoPE-encoded, for any offset $s$ (with valid indices),
\begin{equation*}
\left\langle \tens{Q}_{t_q+s, i, :},\ \tens{K}_{t_k+s, i, :} \right\rangle 
= \left\langle \tens{Q}_{t_q, i, :},\ \tens{K}_{t_k, i, :} \right\rangle.
\end{equation*}
\end{remark}

\subsection{Multi-Head Latent Attention (MLA)}
\label{appendix:mla}
Given a sequence of $n$ tokens with hidden states $\bm{H} \in \mathbb{R}^{n \times d}$, MLA first computes queries as:
\begin{equation*}
\begin{split}
\bm{C}^{\text{Q}} = \alpha_{q}\operatorname{RMSNorm}\left(\bm{H} \bm{W}^{\text{DQ}}\right), \quad 
\overline{\bm{Q}}^{\text{NoPE}} &= \bm{C}^{\text{Q}} \bm{W}^{\text{UQ}}, \quad 
\overline{\bm{Q}}^{\text{RoPE}} = \operatorname{RoPE} \left( \bm{C}^{\text{Q}} \bm{W}^{\text{QR}} \right), \\
\bm{W}^{\text{DQ}} \in \mathbb{R}^{d \times d_c^\prime}, \quad 
\bm{W}^{\text{UQ}} &\in \mathbb{R}^{d_c^\prime \times \left(hd_h\right)}, \quad 
\bm{W}^{\text{QR}} \in \mathbb{R}^{d_c^\prime \times (h d_h^R)}.
\end{split}
\end{equation*}
where $\alpha_{q}=\sqrt{\frac{d}{d_c^{\prime}}}$ is the rescaling factor for query states $\bm{C}^{\text{Q}}$. We then reshape queries to separate heads:
\begin{equation*}
\tens{Q}^{\text{NoPE}} = \operatorname{Reshape}\left(\overline{\bm{Q}}^{\text{NoPE}}, \, \left[n, \, h, \, d_h\right]\right), \quad
\tens{Q}^{\text{RoPE}} = \operatorname{Reshape}\left(\overline{\bm{Q}}^{\text{RoPE}}, \, \left[n, \, h, \, d_h^R\right]\right),
\end{equation*}
where $\tens{Q}^{\text{NoPE}} \in \mathbb{R}^{n \times h \times d_h}$ and $\tens{Q}^{\text{RoPE}} \in \mathbb{R}^{n \times h \times d_h^R}$. These are concatenated along the last dimension to form the final query:
\begin{equation*}
\tens{Q} = \operatorname{Concat}\left(\left[\tens{Q}^{\text{NoPE}}, \, \tens{Q}^{\text{RoPE}}\right], \, \text{dim=2}\right).
\end{equation*}

To reduce the KV cache, MLA obtains shared compressed KV states via a down-projection:
\begin{equation*}
\begin{split}
\bm{C}^{\text{KV}} = \alpha_{kv}\operatorname{RMSNorm}\left(\bm{H} \bm{W}^{\text{DKV}}\right), \quad 
&\bm{W}^{\text{DKV}} \in \mathbb{R}^{d \times d_c},\\
\bm{K}^{\text{RoPE}} = \operatorname{RoPE}\left(\bm{H} \bm{W}^{\text{KR}}\right), \quad 
&\bm{W}^{\text{KR}} \in \mathbb{R}^{d \times d_h^R},
\end{split}
\end{equation*}
where $\alpha_{kv}=\sqrt{\frac{d}{d_c}}$. Both $\bm{C}^{\text{KV}}$ and $\bm{K}^{\text{RoPE}}$ are cached during inference. MLA computes $h$ keys and values using learnable up-projection matrices:
\begin{equation*}
\overline{\bm{K}}^{\text{NoPE}} = \bm{C}^{\text{KV}} \bm{W}^{\text{UK}}, \quad 
\overline{\bm{V}} = \bm{C}^{\text{KV}} \bm{W}^{\text{UV}}, \quad 
\bm{W}^{\text{UK}}, \bm{W}^{\text{UV}} \in \mathbb{R}^{d_c \times \left(hd_h\right)}.
\end{equation*}
These are reshaped into per-head form:
\begin{equation*}
\begin{split}
\tens{K}^{\text{NoPE}} = \operatorname{Reshape}\left(\overline{\bm{K}}^{\text{NoPE}}, \, \left[n, \, h, \, d_h\right]\right)&, \quad 
\tens{K}^{\text{RoPE}} = \operatorname{Reshape}\left(\bm{K}^{\text{RoPE}}, \, \left[n, \, 1, \, d_h^R\right]\right),\\
\tens{V} = \operatorname{Reshape}&\left(\overline{\bm{V}}, \, \left[n, \, h, \, d_h\right]\right),
\end{split}
\end{equation*}
where $\tens{K}^{\text{NoPE}} \in \mathbb{R}^{n\times h \times d_h}$ and $\tens{V}\in\mathbb{R}^{n\times h\times d_h}$.

To obtain per-head position-aware keys, MLA repeats the partial RoPE key across heads:
\begin{equation*}
\tens{K} = \operatorname{Concat}\left(\left[\tens{K}^{\text{NoPE}}, \, \operatorname{RepeatInterleave}\left(\tens{K}^{\text{RoPE}}, \, h, \, \text{dim=1}\right)\right], \, \text{dim=2}\right).
\end{equation*}

\begin{remark}
We analyze the translation equivariance property of MLA. Let $\tens{Q}_{t_q, i, :}=\operatorname{Concat}\left(\tens{Q}^{\text{NoPE}}\left[t_q, \, i, \, :\right], \, \tens{Q}^{\text{RoPE}}\left[t_q, \, i, \, :\right]\right)$ and $\tens{K}_{t_k, i, :}=\operatorname{Concat}\left(\tens{K}^{\text{NoPE}}\left[t_k, \, i, \, :\right], \, \bm{K}^{\text{RoPE}}\left[t_k, \, :\right]\right)$ denote the query and key vectors for head $i$ at positions $t_q$ and $t_k$, respectively. The attention score for this head is given by the inner product:
\begin{equation*}
\left\langle\tens{Q}_{t_q, i, :}, \, \tens{K}_{t_k, i, :}\right\rangle
= \left\langle \tens{Q}^{\text{NoPE}}\left[t_q, \, i, \, :\right], \, \tens{K}^{\text{NoPE}}\left[t_k, \, i, \, :\right] \right\rangle
+ \left\langle \tens{Q}^{\text{RoPE}}\left[t_q, \, i, \, :\right], \, \tens{K}^{\text{RoPE}}\left[t_k, \, :\right] \right\rangle.
\end{equation*}
Between the two terms, the second term with RoPE is position-dependent yet translation equivariant, due to Theorem~\ref{theorem:rope}. The first term is position-independent and thus unchanged under joint translations of $t_q$ and $t_k$. Therefore, the attention score $\left\langle \tens{Q}_{t_q, i, :}, \, \tens{K}_{t_k, i, :} \right\rangle$ is equivariant to translation in input positions. Although MLA introduces positional inductive bias via partial RoPE and is translation equivariant, we refer to this property as semi-translation equivariance to distinguish it from full RoPE translation equivariance.
\end{remark}

\subsection{Multi-matrix Factorization Attention (MFA)}
\label{appendix:mfa}
Given a sequence of $n$ tokens with hidden states $\bm{H}\in\mathbb{R}^{n\times d}$, MFA uses $h$ query heads but only a single shared key-value head (i.e., $g=1$). MFA first projects $\bm{H}$ to a low-rank space and applies RMSNorm:
\begin{equation*}
\bm{C}^{\text{Q}}=\operatorname{RMSNorm}\left(\bm{H}\bm{W}^{\text{CQ}}\right),
\quad
\bm{W}^{\text{CQ}}\in\mathbb{R}^{d\times d_c^\prime}.
\end{equation*}
It then up-projects to all query heads and applies RoPE:
\begin{equation*}
\overline{\bm{Q}}=\operatorname{RoPE}\left(\bm{C}^{\text{Q}}\bm{W}^{\text{UQ}}\right),
\quad
\bm{W}^{\text{UQ}}\in\mathbb{R}^{d_c^\prime\times (h\cdot 2d_h)}.
\end{equation*}
Finally, we reshape into per-head form:
\begin{equation*}
\tens{Q}=\operatorname{Reshape}\left(\overline{\bm{Q}}, \, \left[n, \, h, \, 2d_h\right]\right)\in\mathbb{R}^{n\times h\times 2d_h}.
\end{equation*}

To reduce KV cache, MFA computes only one key head and one value head:
\begin{equation*}
\overline{\bm{K}}^{\text{C}}=\operatorname{RoPE}\left(\bm{H}\bm{W}^{\text{K}}\right),
\quad
\overline{\bm{V}}^{\text{C}}=\bm{H}\bm{W}^{\text{V}},
\quad
\bm{W}^{\text{K}},\bm{W}^{\text{V}}\in\mathbb{R}^{d\times 2d_h},
\end{equation*}
where $\overline{\bm{K}}^{\text{C}},\overline{\bm{V}}^{\text{C}}\in\mathbb{R}^{n\times 2d_h}$. We reshape them as
\begin{equation*}
\tens{K}^{\text{C}}=\operatorname{Reshape}\left(\overline{\bm{K}}^{\text{C}}, \, \left[n, \, 1, \, 2d_h\right]\right),
\quad
\tens{V}^{\text{C}}=\operatorname{Reshape}\left(\overline{\bm{V}}^{\text{C}}, \, \left[n, \, 1, \, 2d_h\right]\right),
\end{equation*}
and cache $\tens{K}^{\text{C}}$ and $\tens{V}^{\text{C}}$ during inference. We repeat them along the head axis to match the $h$ query heads:
\begin{equation*}
\begin{split}
\tens{K}&=\operatorname{RepeatInterleave}\left(\tens{K}^{\text{C}}, \, h, \, \text{dim}=1\right)\in\mathbb{R}^{n\times h\times 2d_h},\\
\tens{V}&=\operatorname{RepeatInterleave}\left(\tens{V}^{\text{C}}, \, h, \, \text{dim}=1\right)\in\mathbb{R}^{n\times h\times 2d_h}.
\end{split}
\end{equation*}
The analysis of translation equivariance is similar to that of MQA.

\subsection{Tensor Product Attention (TPA)}
TPA achieves KV cache compression through low-rank factorization. It represents each head's key/value at a token as a low-rank mixture of $\beta_{kv}$ components: component vectors in $\mathbb{R}^{d_h}$ and head-specific scalar coefficients. During inference, TPA caches the component tensors and coefficient tensors, and computes keys/values on the fly via linear combination.

Given a sequence of $n$ tokens with hidden states $\bm{H} \in \mathbb{R}^{n \times d}$, TPA first computes the query/key/value factors:
\begin{equation*}
\begin{split}
\overline{\bm{Q}}^{\text{A}} &= \bm{H}\bm{W}^{\text{AQ}}, \quad
\bm{W}^{\text{AQ}} \in \mathbb{R}^{d \times (\beta_q h)}, \quad
\overline{\bm{Q}}^{\text{A}} \in \mathbb{R}^{n \times (\beta_q h)},\\
\overline{\bm{Q}}^{\text{C}} &= \bm{H}\bm{W}^{\text{CQ}}, \quad
\bm{W}^{\text{CQ}} \in \mathbb{R}^{d \times (\beta_q d_h)}, \quad
\overline{\bm{Q}}^{\text{C}} \in \mathbb{R}^{n \times (\beta_q d_h)},\\
\overline{\bm{K}}^{\text{A}} &= \bm{H}\bm{W}^{\text{AK}}, \quad
\bm{W}^{\text{AK}} \in \mathbb{R}^{d \times (\beta_{kv} h)}, \quad
\overline{\bm{K}}^{\text{A}} \in \mathbb{R}^{n \times (\beta_{kv} h)},\\
\overline{\bm{K}}^{\text{C}} &= \bm{H}\bm{W}^{\text{CK}}, \quad
\bm{W}^{\text{CK}} \in \mathbb{R}^{d \times (\beta_{kv} d_h)}, \quad
\overline{\bm{K}}^{\text{C}} \in \mathbb{R}^{n \times (\beta_{kv} d_h)},\\
\overline{\bm{V}}^{\text{A}} &= \bm{H}\bm{W}^{\text{AV}}, \quad
\bm{W}^{\text{AV}} \in \mathbb{R}^{d \times (\beta_{kv} h)}, \quad
\overline{\bm{V}}^{\text{A}} \in \mathbb{R}^{n \times (\beta_{kv} h)},\\
\overline{\bm{V}}^{\text{C}} &= \bm{H}\bm{W}^{\text{CV}}, \quad
\bm{W}^{\text{CV}} \in \mathbb{R}^{d \times (\beta_{kv} d_h)}, \quad
\overline{\bm{V}}^{\text{C}} \in \mathbb{R}^{n \times (\beta_{kv} d_h)}.
\end{split}
\end{equation*}

We reshape the projections into 3D tensors:
\begin{equation*}
\begin{split}
\tens{Q}^{\text{A}} &= \operatorname{Reshape}\left(\overline{\bm{Q}}^{\text{A}}, \, \left[n, \, \beta_q, \, h\right]\right), \\
\tens{Q}^{\text{C}} &= \operatorname{RoPE}\left(\operatorname{Reshape}\left(\overline{\bm{Q}}^{\text{C}}, \, \left[n, \, \beta_q, \, d_h\right]\right)\right),\\
\tens{K}^{\text{A}} &= \operatorname{Reshape}\left(\overline{\bm{K}}^{\text{A}}, \, \left[n, \, \beta_{kv}, \, h\right]\right), \\
\tens{K}^{\text{C}} &= \operatorname{RoPE}\left(\operatorname{Reshape}\left(\overline{\bm{K}}^{\text{C}}, \, \left[n, \, \beta_{kv}, \, d_h\right]\right)\right),\\
\tens{V}^{\text{A}} &= \operatorname{Reshape}\left(\overline{\bm{V}}^{\text{A}}, \, \left[n, \, \beta_{kv}, \, h\right]\right), \\
\tens{V}^{\text{C}} &= \operatorname{Reshape}\left(\overline{\bm{V}}^{\text{C}}, \, \left[n, \, \beta_{kv}, \, d_h\right]\right),
\end{split}
\end{equation*}
so that $\tens{Q}^{\text{A}} \in \mathbb{R}^{n \times \beta_q \times h}$, $\tens{Q}^{\text{C}} \in \mathbb{R}^{n \times \beta_q \times d_h}$, $\tens{K}^{\text{A}} \in \mathbb{R}^{n \times \beta_{kv} \times h}$, $\tens{K}^{\text{C}} \in \mathbb{R}^{n \times \beta_{kv} \times d_h}$, $\tens{V}^{\text{A}} \in \mathbb{R}^{n \times \beta_{kv} \times h}$, $\tens{V}^{\text{C}} \in \mathbb{R}^{n \times \beta_{kv} \times d_h}$.

For each token position $t\in\{0,\ldots,n-1\}$, the final query, key, and value matrices are computed as:
\begin{equation*}
\begin{split}
\tens{Q}\left[t, \, :, \, :\right] &= \frac{1}{\beta_q} \left(\tens{Q}^{\text{A}}\left[t, \, :, \, :\right]\right)^{\top} \tens{Q}^{\text{C}}\left[t, \, :, \, :\right] \in \mathbb{R}^{h \times d_h},\\
\tens{K}\left[t, \, :, \, :\right] &= \frac{1}{\beta_{kv}} \left(\tens{K}^{\text{A}}\left[t, \, :, \, :\right]\right)^{\top} \tens{K}^{\text{C}}\left[t, \, :, \, :\right] \in \mathbb{R}^{h \times d_h},\\
\tens{V}\left[t, \, :, \, :\right] &= \frac{1}{\beta_{kv}} \left(\tens{V}^{\text{A}}\left[t, \, :, \, :\right]\right)^{\top} \tens{V}^{\text{C}}\left[t, \, :, \, :\right] \in \mathbb{R}^{h \times d_h}.
\end{split}
\end{equation*}
During inference, TPA caches $\tens{K}^{\text{A}}, \tens{K}^{\text{C}}, \tens{V}^{\text{A}}, \tens{V}^{\text{C}}$.

\begin{remark} 
Fix a head index $i$ and token positions $t_q,t_k$. Let $\tens{Q}_{t_q, i, :} := \tens{Q}\left[t_q, \, i, \, :\right]\in\mathbb{R}^{d_h}$ and $\tens{K}_{t_k, i, :} := \tens{K}\left[t_k, \, i, \, :\right]\in\mathbb{R}^{d_h}$. From the computation above, we have
\begin{equation*}
\small
\tens{Q}_{t_q, i, :} = \frac{1}{\beta_q}\sum_{b_q=0}^{\beta_q-1} \tens{Q}^{\text{A}}\left[t_q, \, b_q, \, i\right]\ \tens{Q}^{\text{C}}\left[t_q, \, b_q, \, :\right],
\qquad
\tens{K}_{t_k, i, :} = \frac{1}{\beta_{kv}}\sum_{b_{kv}=0}^{\beta_{kv}-1} \tens{K}^{\text{A}}\left[t_k, \, b_{kv}, \, i\right]\ \tens{K}^{\text{C}}\left[t_k, \, b_{kv}, \, :\right].
\end{equation*}
Therefore, the inner product expands as
\begin{equation*}
\small
\begin{split}
\left\langle \tens{Q}_{t_q, i, :}, \, \tens{K}_{t_k, i, :} \right\rangle
&= \frac{1}{\beta_q \beta_{kv}}
\sum_{b_q=0}^{\beta_q-1}\sum_{b_{kv}=0}^{\beta_{kv}-1}
\tens{Q}^{\text{A}}\left[t_q, \, b_q, \, i\right] \tens{K}^{\text{A}}\left[t_k, \, b_{kv}, \, i\right]
\left\langle \tens{Q}^{\text{C}}\left[t_q, \, b_q, \, :\right], \, \tens{K}^{\text{C}}\left[t_k, \, b_{kv}, \, :\right] \right\rangle.
\end{split}
\end{equation*}

Since $\tens{Q}^{\text{C}}$ and $\tens{K}^{\text{C}}$ are RoPE-encoded, Theorem~\ref{theorem:rope} implies that for any offset $s$,
\begin{equation*}
\left\langle \tens{Q}^{\text{C}}\left[t_q+s, \, b_q, \, :\right],\ \tens{K}^{\text{C}}\left[t_k+s, \, b_{kv}, \, :\right] \right\rangle
=
\left\langle \tens{Q}^{\text{C}}\left[t_q, \, b_q, \, :\right],\ \tens{K}^{\text{C}}\left[t_k, \, b_{kv}, \, :\right] \right\rangle,
\quad
\forall\, b_q,\ b_{kv}.
\end{equation*}
Because the scalar coefficients $\tens{Q}^{\text{A}}\left[t_q, \, b_q, \, i\right]$ and $\tens{K}^{\text{A}}\left[t_k, \, b_{kv}, \, i\right]$ shift with the tokens under translation, the full double-sum is equivariant under jointly translating $t_q$ and $t_k$ by the same offset $s$:
\begin{equation*}
\left\langle \tens{Q}_{t_q+s, i, :}, \, \tens{K}_{t_k+s, i, :} \right\rangle
=
\left\langle \tens{Q}_{t_q, i, :}, \, \tens{K}_{t_k, i, :} \right\rangle.
\end{equation*}
Thus, TPA preserves translation equivariance of attention scores with RoPE.
\end{remark}
\subsection{Grouped Latent Attention (GLA)}
Given a sequence of $n$ tokens with hidden states $\bm{H} \in \mathbb{R}^{n \times d}$, GLA divides the $h$ attention heads into $g$ groups (e.g., $g=2$), where each group has $r = h/g$ heads. GLA adopts the same query computation mechanism as MLA.

Instead of a single compressed KV state, GLA computes $g$ independent compressed states:
\begin{equation*}
\bm{C}^{j, \text{KV}} = \alpha_{kv}\operatorname{RMSNorm}\left(\bm{H} \bm{W}^{j, \text{DKV}}\right), \quad 
\bm{W}^{j, \text{DKV}} \in \mathbb{R}^{d \times (d_c/g)},
\end{equation*}
where $j \in \{0, \ldots, g-1\}$, $d_c$ is the total latent dimension, each group uses $d_c/g$ dimension, and $\alpha_{kv}=\sqrt{\frac{gd}{d_c}}$.

The RoPE keys remain shared across all groups:
\begin{equation*}
\bm{K}^{\text{RoPE}} = \operatorname{RoPE}\left(\bm{H} \bm{W}^{\text{KR}}\right), \quad 
\bm{W}^{\text{KR}} \in \mathbb{R}^{d \times d_h^R}.
\end{equation*}

During inference, we cache $\left\{\bm{C}^{j, \text{KV}}, \ldots, \bm{C}^{g-1, \text{KV}}\right\}$ and $\bm{K}^{\text{RoPE}}$ with total KV cache size of $d_c + d_h^R$ per token.

Each group independently computes its keys and values:
\begin{equation*}
\overline{\bm{K}}^{j, \text{NoPE}} = \bm{C}^{j, \text{KV}} \bm{W}^{j, \text{UK}}, \quad \overline{\bm{V}}^{j} = \bm{C}^{j, \text{KV}} \bm{W}^{j, \text{UV}}, \quad \bm{W}^{j, \text{UK}}, \bm{W}^{j, \text{UV}} \in \mathbb{R}^{(d_c/g) \times (r d_h)},
\end{equation*}
where $r = h/g$ is the number of heads per group.

Reshape into per-head form for each group:
\begin{equation*}
\begin{split}
\tens{K}^{j, \text{NoPE}} = \operatorname{Reshape}\left(\overline{\bm{K}}^{j, \text{NoPE}}, \, \left[n, \, r, \, d_h\right]\right),& \quad 
\tens{V}^{j} = \operatorname{Reshape}\left(\overline{\bm{V}}^{j}, \, \left[n, \, r, \, d_h\right]\right), \\ \tens{K}^{\text{RoPE}} = \operatorname{Reshape}&\left(\bm{K}^{\text{RoPE}}, \, \left[n, \, 1, \, d_h^R\right]\right)
\end{split}
\end{equation*}

Construct position-aware keys for each group by repeating the shared RoPE keys:
\begin{equation*}
\tens{K}^{j} = \operatorname{Concat}\left(\left[\tens{K}^{j, \text{NoPE}}, \,  \operatorname{RepeatInterleave}\left(\tens{K}^{\text{RoPE}}, \, r, \, \text{dim=1}\right) \right], \, \text{dim=2}\right)
\in \mathbb{R}^{n \times r \times (d_h + d_h^R)}.
\end{equation*}

We finally concatenate all $\tens{K}^{j}$, $\tens{V}^{j}$ to obtain:
\begin{equation*}
    \tens{K} = \operatorname{Concat}\left(\left[\tens{K}^{0}, \, \ldots, \, \tens{K}^{g-1} \right], \, \text{dim=1}\right), \quad \tens{V} = \operatorname{Concat}\left(\left[\tens{V}^{0}, \, \ldots, \, \tens{V}^{g-1} \right], \, \text{dim=1}\right),
\end{equation*}
where $\tens{K}\in\mathbb{R}^{n\times h\times (d_h+d_h^R)}$ and $\tens{V}\in\mathbb{R}^{n\times h\times d_h}$. The analysis of translation equivariance is similar to that of MLA.

\subsection{Grouped-Tied Attention (GTA)}
\label{appendix:gta}
Given a sequence of $n$ tokens with hidden states $\bm{H}\in\mathbb{R}^{n\times d}$, GTA uses $h$ query heads and $g$ value heads, and computes queries as
\begin{equation*}
\overline{\bm{Q}}=\bm{H}\bm{W}^{\text{Q}},
\quad
\bm{W}^{\text{Q}}\in\mathbb{R}^{d\times (hd_h)},
\quad
\widetilde{\tens{Q}}=\operatorname{Reshape}\left(\overline{\bm{Q}}, \, \left[n, \, h, \, d_h\right]\right)\in\mathbb{R}^{n\times h\times d_h}.
\end{equation*}
We split $\widetilde{\bm{Q}}$ into NoPE and RoPE parts and apply RoPE to the latter:
\begin{equation*}
\begin{split}
\tens{Q}^{\text{NoPE}} &= \widetilde{\tens{Q}}\left[:, \, :, \, :d_h-d_h^R\right]\in\mathbb{R}^{n\times h\times \left(d_h-d_h^R\right)},
\\
\tens{Q}^{\text{RoPE}} &= \operatorname{RoPE}\left(\widetilde{\tens{Q}}\left[:, \, :, \, d_h-d_h^R:\right]\right)\in\mathbb{R}^{n\times h\times d_h^R}, \\
\tens{Q} &= \operatorname{Concat}\left(\left[\tens{Q}^{\text{NoPE}}, \, \tens{Q}^{\text{RoPE}}\right], \, \text{dim=2}\right)\in\mathbb{R}^{n\times h\times d_h}.
\end{split}
\end{equation*}

GTA computes a single RoPE key shared across all heads:
\begin{equation*}
\begin{split}
\overline{\bm{K}}^{\text{RoPE}}=\bm{H}\bm{W}^{\text{KR}}, \quad &
\bm{W}^{\text{KR}}\in\mathbb{R}^{d\times d_h^R}, \quad
\bm{K}^{\text{RoPE}}=\operatorname{RoPE}\left(\overline{\bm{K}}^{\text{RoPE}}\right), \\
\tens{K}^{\text{RoPE}} &= \operatorname{Reshape}\left(\bm{K}^{\text{RoPE}}, \, \left[n, \, 1, \, d_h^R\right]\right)
\end{split}
\end{equation*}

To reduce KV cache, GTA computes grouped value states with only $g$ heads:
\begin{equation*}
\overline{\bm{V}}=\bm{H}\bm{W}^{\text{KV}},
\quad
\bm{W}^{\text{KV}}\in\mathbb{R}^{d\times (gd_h)},
\quad
\tens{V}^{\text{C}}=\operatorname{Reshape}\left(\overline{\bm{V}}, \, \left[n, \, g, \, d_h\right]\right).
\end{equation*}
We cache $\tens{V}^{\text{C}}$ and $\bm{K}^{\text{RoPE}}$ during inference. We repeat $\tens{V}^{\text{C}}$ along the head axis with $r=h/g$ to form the final values:
\begin{equation*}
\tens{V}=\operatorname{RepeatInterleave}\left(\tens{V}^{\text{C}}, \, r, \, \text{dim}=1\right)\in\mathbb{R}^{n\times h\times d_h}.
\end{equation*}

We then form the final keys by tying the NoPE part of keys to the values and concatenating with the shared RoPE key:
\begin{equation*}
\small
\tens{K}=\operatorname{Concat}\left(\left[\tens{V}\left[:, \, :, \, :d_h-d_h^R\right],\ \operatorname{RepeatInterleave}\left(\tens{K}^{\text{RoPE}}, \, h, \, \text{dim=1}\right)\right], \, \text{dim=2}\right)\in\mathbb{R}^{n\times h\times d_h}.
\end{equation*}
The analysis of translation equivariance is similar to that of MLA.

\section{Llama-3 Architecture}
\label{appendix:llama3}
Given hidden states $\bm H\in\mathbb R^{n\times d}$ for a sequence of $n$ tokens, we first compute the attention output
\begin{equation*}
\begin{split}
\bm{H}^{\prime}&=\operatorname{RMSNorm}\left(\bm{H}\right)\\
\bm{O}^{\text{attn}}&=\operatorname{Attention}\left(\bm{H}^{\prime}\right)\in\mathbb R^{n\times \left(h d_h\right)},
\end{split}
\end{equation*}
then project back to the model dimension and add a residual:
\begin{equation*}
\bm{H} \leftarrow \bm{H} + \bm{O}^{\text{attn}}\bm{W}^{\text{O},\text{attn}},\qquad 
\bm{W}^{\text{O},\text{attn}}\in\mathbb{R}^{\left(h d_h\right)\times d}.
\end{equation*}
Next, an MLP block (gated form) is applied:
\begin{equation*}
\begin{split}
\bm{H}^{\prime}&=\operatorname{RMSNorm}\left(\bm{H}\right)\\
\bm{O}^{\text{mlp}}=\sigma\left(\bm{H}^{\prime} \bm{W}^1\right) &\odot\left(\bm{H}^{\prime} \bm{W}^2\right),\qquad 
\bm{W}^1,\bm{W}^2\in\mathbb R^{d\times d_f},
\end{split}
\end{equation*}
followed by the output projection and residual:
\begin{equation*}
\bm{H} \leftarrow \bm{H} + \bm{O}^{\text{mlp}}\bm{W}^{\text{O},\text{mlp}},\qquad 
\bm{W}^{\text{O},\text{mlp}}\in\mathbb R^{d_f\times d},
\end{equation*}
where $\sigma(\cdot)$ is an elementwise nonlinearity function such as SiLU and $\odot$ denotes elementwise multiplication.

\section{Gated Attention}
\label{appendix:gated_attention}
Given hidden states $\bm H\in\mathbb R^{n\times d}$ for a sequence of $n$ tokens, we first compute the attention output with gated score
\begin{equation*}
\begin{split}
\bm{G}&=\varsigma\left(\bm{H}\bm{W}^{\text{G}}\right)\\
\bm{H}^{\prime}&=\operatorname{RMSNorm}\left(\bm{H}\right)\\
\bm{O}^{\text{attn}}&=\operatorname{Attention}\left(\bm{H}^{\prime}\right)\odot\bm{G}\in\mathbb R^{n\times \left(h d_h\right)},
\end{split}
\end{equation*}
then project back to the model dimension and add a residual:
\begin{equation*}
\bm{H} \leftarrow \bm{H} + \bm{O}^{\text{attn}}\bm{W}^{\text{O},\text{attn}},\qquad 
\bm{W}^{\text{O},\text{attn}}\in\mathbb{R}^{\left(h d_h\right)\times d}.
\end{equation*}
Next, an MLP block (gated form) is applied:
\begin{equation*}
\begin{split}
\bm{H}^{\prime}&=\operatorname{RMSNorm}\left(\bm{H}\right)\\
\bm{O}^{\text{mlp}}=\sigma\left(\bm{H}^{\prime} \bm{W}^1\right) &\odot\left(\bm{H}^{\prime} \bm{W}^2\right),\qquad 
\bm{W}^1,\bm{W}^2\in\mathbb R^{d\times d_f},
\end{split}
\end{equation*}
followed by the output projection and residual:
\begin{equation*}
\bm{H} \leftarrow \bm{H} + \bm{O}^{\text{mlp}}\bm{W}^{\text{O},\text{mlp}},\qquad 
\bm{W}^{\text{O},\text{mlp}}\in\mathbb R^{d_f\times d},
\end{equation*}
where $\varsigma(\cdot)$ is an elementwise nonlinearity function such as sigmoid and $\odot$ denotes elementwise multiplication.
\section{Architectural Hyperparameters}
\subsection{Architectural Hyperparameters for Main Results}
\label{appendix:main_parameters}
Our model is based on the Llama-3 architecture, adopting a configuration largely consistent with Llama-3.2-3B but modified to 24 layers (down from 28). The architecture utilizes 24 attention heads, a model hidden dimension ($d$) of 3072, a head dimension ($d_h$) of 128, and an intermediate Feedforward Network (FFN) dimension ($d_f$) of 8192. The architectural hyperparameters for our baselines are aligned with their original implementations. Specifically, MLA is configured with latent dimensions $d_c^{\prime}=12d_h$, $d_c=4d_h$, and $d_h^R=0.5d_h$; GLA, along with our proposed MLRA, adopts $d_c^{\prime}=8d_h$, $d_c=4d_h$, and $d_h^R=0.5d_h$; and TPA uses ranks $\beta_{q}=6$ and $\beta_{kv}=2$. For GQA and GTA, we set the number of KV heads to $g=h/4$. We report the detailed architectural hyperparameters for our main experiments in Tables~\ref{tab:arch_mha_main},~\ref{tab:arch_mqa_main},~\ref{tab:arch_gqa_main},~\ref{tab:arch_mla_main},~\ref{tab:arch_mfa_main},~\ref{tab:arch_tpa_main},~\ref{tab:arch_gla_2_main},~\ref{tab:arch_gla_4_main},~\ref{tab:arch_gta_main},~\ref{tab:arch_mlra_2_main}, and~\ref{tab:arch_mlra_4_main}.
\begin{table}[H]
\caption{Model configuration of MHA for main results.}
\vspace{-0.1in}
\centering
\resizebox{\linewidth}{!}{
\setlength{\tabcolsep}{15pt}
\begin{tabular}{c | c c c c c c}
\toprule
\textbf{Model Size} & \textbf{\# Parameters} & \textbf{\# Layers} & \boldmath{$h$} & \boldmath{$d$}  & \boldmath{$d_h$} & \boldmath{$d_f$}\\
\midrule
2.9B & 2872.59M  & 24 & 24 & 3072 & 128 & 8192 \\
\bottomrule
\end{tabular}
}
\label{tab:arch_mha_main}
\vspace{-0.1in}
\end{table}

\begin{table}[H]
\caption{Model configuration of MQA for main results.}
\vspace{-0.1in}
\resizebox{\linewidth}{!}{
\centering
\setlength{\tabcolsep}{12pt}
\begin{tabular}{c | c c c c c c c}
\toprule
\textbf{Model Size} & \textbf{\# Parameters} & \textbf{\# Layers} & \boldmath{$h$} & \boldmath{$g$} & \boldmath{$d$}  & \boldmath{$d_h$} & \boldmath{$d_f$}\\
\midrule
2.9B & 2872.00M  & 24 & 24 & 1 & 3072 & 128 & 10152 \\
\bottomrule
\end{tabular}
}
\label{tab:arch_mqa_main}
\vspace{-0.1in}
\end{table}

\begin{table}[H]
\caption{Model configuration of GQA for main results.}
\vspace{-0.1in}
\resizebox{\linewidth}{!}{
\centering
\setlength{\tabcolsep}{12pt}
\begin{tabular}{c | c c c c c c c}
\toprule
\textbf{Model Size} & \textbf{\# Parameters} & \textbf{\# Layers} & \boldmath{$h$} & \boldmath{$g$} & \boldmath{$d$}  & \boldmath{$d_h$} & \boldmath{$d_f$}\\
\midrule
2.9B & 2872.59M  & 24 & 24 & 6 & 3072 & 128 & 9728 \\
\bottomrule
\end{tabular}
}
\label{tab:arch_gqa_main}
\vspace{-0.1in}
\end{table}

\begin{table}[H]
\caption{Model configuration of MLA for main results.}
\vspace{-0.1in}
\resizebox{\linewidth}{!}{
\centering
\setlength{\tabcolsep}{5pt}
\begin{tabular}{c | c c c c c c c c c c c}
\toprule
\textbf{Model Size} & \textbf{\# Parameters} & \textbf{\# Layers} & \boldmath{$h$} & \boldmath{$d_c^{\prime}$}  & \boldmath{$d_c$} & \boldmath{$d$} & \boldmath{$\alpha_q$} & \boldmath{$\alpha_{kv}$} & \boldmath{$d_h$} & \boldmath{$d_h^R$} & \boldmath{$d_f$} \\
\midrule
2.9B & 2872.05M  & 24 & 24 & 1536 & 512 & 3072 & $\sqrt{2}$ & $\sqrt{6}$ & 128 & 64 & 9448 \\
\bottomrule
\end{tabular}
}
\label{tab:arch_mla_main}
\vspace{-0.1in}
\end{table}

\begin{table}[H]
\caption{Model configuration of MFA for main results.}
\vspace{-0.1in}
\resizebox{\linewidth}{!}{
\centering
\setlength{\tabcolsep}{10pt}
\begin{tabular}{c | c c c c c c c c}
\toprule
\textbf{Model Size} & \textbf{\# Parameters} & \textbf{\# Layers} & \boldmath{$h$} & \boldmath{$g$} & \boldmath{$d_c^{\prime}$} & \boldmath{$d$}  & \boldmath{$d_h$} & \boldmath{$d_f$}\\
\midrule
2.9B & 2873.23M  & 24 & 24 & 1 & 2048 & 3072 & 256 & 8024 \\
\bottomrule
\end{tabular}
}
\label{tab:arch_mfa_main}
\vspace{-0.1in}
\end{table}

\begin{table}[H]
\caption{Model configuration of TPA for main results.}
\vspace{-0.1in}
\resizebox{\linewidth}{!}{
\centering
\setlength{\tabcolsep}{10pt}
\begin{tabular}{c | c c c c c c c c}
\toprule
\textbf{Model Size} & \textbf{\# Parameters} & \textbf{\# Layers} & \boldmath{$h$} & \boldmath{$\beta_{q}$} & \boldmath{$\beta_{kv}$} & \boldmath{$d$}  & \boldmath{$d_h$} & \boldmath{$d_f$}\\
\midrule
2.9B & 2873.18M  & 24 & 24 & 6 & 2 & 3072 & 128 & 10760 \\
\bottomrule
\end{tabular}
}
\label{tab:arch_tpa_main}
\vspace{-0.1in}
\end{table}

\begin{table}[H]
\caption{Model configuration of GLA-2 for main results.}
\vspace{-0.1in}
\resizebox{\linewidth}{!}{
\centering
\setlength{\tabcolsep}{3pt}
\begin{tabular}{c | c c c c c c c c c c c c}
\toprule
\textbf{Model Size} & \textbf{\# Parameters} & \textbf{\# Layers} & \boldmath{$h$} & \boldmath{$g$} & \boldmath{$d_c^{\prime}$} & \boldmath{$d_c$} & \boldmath{$d$} & \boldmath{$\alpha_q$} & \boldmath{$\alpha_{kv}$} & \boldmath{$d_h$} & \boldmath{$d_h^R$} & \boldmath{$d_f$} \\
\midrule
2.9B & 2872.63M  & 24 & 24 & 2 & 1024 & 512 & 3072 & $\sqrt{3}$ & $\sqrt{12}$ & 128 & 64 & 10048 \\
\bottomrule
\end{tabular}
}
\label{tab:arch_gla_2_main}
\vspace{-0.1in}
\end{table}

\begin{table}[H]
\caption{Model configuration of GLA-4 for main results.}
\vspace{-0.1in}
\resizebox{\linewidth}{!}{
\centering
\setlength{\tabcolsep}{5pt}
\begin{tabular}{c | c c c c c c c c c c c c}
\toprule
\textbf{Model Size} & \textbf{\# Parameters} & \textbf{\# Layers} & \boldmath{$h$} & \boldmath{$g$} & \boldmath{$d_c^{\prime}$} & \boldmath{$d_c$} & \boldmath{$d$} & \boldmath{$\alpha_q$} & \boldmath{$\alpha_{kv}$} & \boldmath{$d_h$} & \boldmath{$d_h^R$} & \boldmath{$d_f$} \\
\midrule
2.9B & 2873.22M  & 24 & 24 & 4 & 1024 & 512 & 3072 & $\sqrt{3}$ & $\sqrt{24}$  & 128 & 64 & 10136 \\
\bottomrule
\end{tabular}
}
\label{tab:arch_gla_4_main}
\vspace{-0.1in}
\end{table}

\begin{table}[H]
\caption{Model configuration of GTA for main results.}
\vspace{-0.1in}
\resizebox{\linewidth}{!}{
\centering
\setlength{\tabcolsep}{10pt}
\begin{tabular}{c | c c c c c c c c}
\toprule
\textbf{Model Size} & \textbf{\# Parameters} & \textbf{\# Layers} & \boldmath{$h$} & \boldmath{$g$} & \boldmath{$d$}  & \boldmath{$d_h$} & \boldmath{$d_h^R$} & \boldmath{$d_f$}\\
\midrule
2.9B & 2872.00M  & 24 & 24 & 6 & 3072 & 128 & 64 & 9960 \\
\bottomrule
\end{tabular}
}
\label{tab:arch_gta_main}
\vspace{-0.1in}
\end{table}

\begin{table}[H]
\caption{Model configuration of MLRA-2 for main results.}
\vspace{-0.1in}
\resizebox{\linewidth}{!}{
\centering
\setlength{\tabcolsep}{3pt}
\begin{tabular}{c | c c c c c c c c c c c c}
\toprule
\textbf{Model Size} & \textbf{\# Parameters} & \textbf{\# Layers} & \boldmath{$h$} & \boldmath{$d_c^{\prime}$} & \boldmath{$d_c$} & \boldmath{$d$} & \boldmath{$\alpha_q$} & \boldmath{$\alpha_{kv}$} & \boldmath{$\alpha_{attn}$} & \boldmath{$d_h$} & \boldmath{$d_h^R$} & \boldmath{$d_f$}  \\
\midrule
2.9B & 2872.63M  & 24 & 24 & 1024 & 512 & 3072 & $\sqrt{3}$ & $\sqrt{24}$ & $\frac{\sqrt{2}}{2}$  & 128 & 64 & 10048 \\
\bottomrule
\end{tabular}
}
\label{tab:arch_mlra_2_main}
\vspace{-0.1in}
\end{table}

\begin{table}[H]
\caption{Model configuration of MLRA-4 for main results.}
\vspace{-0.1in}
\resizebox{\linewidth}{!}{
\centering
\setlength{\tabcolsep}{3pt}
\begin{tabular}{c | c c c c c c c c c c c c}
\toprule
\textbf{Model Size} & \textbf{\# Parameters} & \textbf{\# Layers} & \boldmath{$h$} & \boldmath{$d_c^{\prime}$} & \boldmath{$d_c$} & \boldmath{$d$} & \boldmath{$\alpha_q$} & \boldmath{$\alpha_{kv}$} & \boldmath{$\alpha_{attn}$} & \boldmath{$d_h$} & \boldmath{$d_h^R$} & \boldmath{$d_f$}   \\
\midrule
2.9B & 2873.22M  & 24 & 24 & 1024 & 512 & 3072 & $\sqrt{3}$ & $\sqrt{24}$ & $\frac{1}{2}$ & 128 & 64 & 9880  \\
\bottomrule
\end{tabular}
}
\label{tab:arch_mlra_4_main}
\end{table}
\subsection{Architectural Hyperparameters for Initialization Ablation Study}
\label{appendix:initialization_parameters}
In our initialization ablation study, we focus on the initialization of the attention and FFN output projections parameters ($\bm{W}^{\text{O, attn}}, \bm{W}^{\text{O, mlp}}$). We evaluate two distinct initialization strategies: zero initialization versus a Gaussian distribution $\mathcal{N}(0, \sigma=0.02)$, to identify which yields better performance. To isolate the impact of the initialization strategy, the model architecture and all other hyperparameters are kept identical to those used for our main results. We report the detailed architectural hyperparameters for our initialization ablation in Tables~\ref{tab:arch_mha_initialization},~\ref{tab:arch_mqa_initialization},~\ref{tab:arch_gqa_initialization},~\ref{tab:arch_mla_initialization},~\ref{tab:arch_mfa_initialization},~\ref{tab:arch_tpa_initialization},~\ref{tab:arch_gla_2_initialization},~\ref{tab:arch_gla_4_initialization}, and~\ref{tab:arch_gta_initialization}.
\begin{table}[H]
\caption{Model configuration of MHA for initialization ablation.}
\vspace{-0.1in}
\centering
\resizebox{\linewidth}{!}{
\setlength{\tabcolsep}{15pt}
\begin{tabular}{c | c c c c c c}
\toprule
\textbf{Model Size} & \textbf{\# Parameters} & \textbf{\# Layers} & \boldmath{$h$} & \boldmath{$d$}  & \boldmath{$d_h$} & \boldmath{$d_f$}\\
\midrule
2.9B & 2872.59M  & 24 & 24 & 3072 & 128 & 8192 \\
\bottomrule
\end{tabular}
}
\label{tab:arch_mha_initialization}
\vspace{-0.1in}
\end{table}

\begin{table}[H]
\caption{Model configuration of MQA for initialization ablation.}
\vspace{-0.1in}
\resizebox{\linewidth}{!}{
\centering
\setlength{\tabcolsep}{12pt}
\begin{tabular}{c | c c c c c c c}
\toprule
\textbf{Model Size} & \textbf{\# Parameters} & \textbf{\# Layers} & \boldmath{$h$} & \boldmath{$g$} & \boldmath{$d$}  & \boldmath{$d_h$} & \boldmath{$d_f$}\\
\midrule
2.9B & 2872.00M  & 24 & 24 & 1 & 3072 & 128 & 10152 \\
\bottomrule
\end{tabular}
}
\label{tab:arch_mqa_initialization}
\vspace{-0.1in}
\end{table}

\begin{table}[H]
\caption{Model configuration of GQA for initialization ablation.}
\vspace{-0.1in}
\resizebox{\linewidth}{!}{
\centering
\setlength{\tabcolsep}{12pt}
\begin{tabular}{c | c c c c c c c}
\toprule
\textbf{Model Size} & \textbf{\# Parameters} & \textbf{\# Layers} & \boldmath{$h$} & \boldmath{$g$} & \boldmath{$d$}  & \boldmath{$d_h$} & \boldmath{$d_f$}\\
\midrule
2.9B & 2872.59M  & 24 & 24 & 6 & 3072 & 128 & 9728 \\
\bottomrule
\end{tabular}
}
\label{tab:arch_gqa_initialization}
\vspace{-0.1in}
\end{table}

\begin{table}[H]
\caption{Model configuration of MLA for initialization ablation.}
\vspace{-0.1in}
\resizebox{\linewidth}{!}{
\centering
\setlength{\tabcolsep}{5pt}
\begin{tabular}{c | c c c c c c c c c c c}
\toprule
\textbf{Model Size} & \textbf{\# Parameters} & \textbf{\# Layers} & \boldmath{$h$} & \boldmath{$d_c^{\prime}$}  & \boldmath{$d_c$} & \boldmath{$d$} & \boldmath{$\alpha_q$} & \boldmath{$\alpha_{kv}$} & \boldmath{$d_h$} & \boldmath{$d_h^R$} & \boldmath{$d_f$} \\
\midrule
2.9B & 2872.05M  & 24 & 24 & 1536 & 512 & 3072 & $\sqrt{2}$ & $\sqrt{6}$ & 128 & 64 & 9448 \\
\bottomrule
\end{tabular}
}
\label{tab:arch_mla_initialization}
\vspace{-0.1in}
\end{table}

\begin{table}[H]
\caption{Model configuration of MFA for initialization ablation.}
\vspace{-0.1in}
\resizebox{\linewidth}{!}{
\centering
\setlength{\tabcolsep}{10pt}
\begin{tabular}{c | c c c c c c c c}
\toprule
\textbf{Model Size} & \textbf{\# Parameters} & \textbf{\# Layers} & \boldmath{$h$} & \boldmath{$g$} & \boldmath{$d_c^{\prime}$} & \boldmath{$d$}  & \boldmath{$d_h$} & \boldmath{$d_f$}\\
\midrule
2.9B & 2873.23M  & 24 & 24 & 1 & 2048 & 3072 & 256 & 8024 \\
\bottomrule
\end{tabular}
}
\label{tab:arch_mfa_initialization}
\vspace{-0.1in}
\end{table}

\begin{table}[H]
\caption{Model configuration of TPA for initialization ablation.}
\vspace{-0.1in}
\resizebox{\linewidth}{!}{
\centering
\setlength{\tabcolsep}{10pt}
\begin{tabular}{c | c c c c c c c c}
\toprule
\textbf{Model Size} & \textbf{\# Parameters} & \textbf{\# Layers} & \boldmath{$h$} & \boldmath{$\beta_{q}$} & \boldmath{$\beta_{kv}$} & \boldmath{$d$}  & \boldmath{$d_h$} & \boldmath{$d_f$}\\
\midrule
2.9B & 2873.18M  & 24 & 24 & 6 & 2 & 3072 & 128 & 10760 \\
\bottomrule
\end{tabular}
}
\label{tab:arch_tpa_initialization}
\vspace{-0.1in}
\end{table}

\begin{table}[H]
\caption{Model configuration of GLA-2 for initialization ablation.}
\vspace{-0.1in}
\resizebox{\linewidth}{!}{
\centering
\setlength{\tabcolsep}{3pt}
\begin{tabular}{c | c c c c c c c c c c c c}
\toprule
\textbf{Model Size} & \textbf{\# Parameters} & \textbf{\# Layers} & \boldmath{$h$} & \boldmath{$g$} & \boldmath{$d_c^{\prime}$} & \boldmath{$d_c$} & \boldmath{$d$} & \boldmath{$\alpha_q$} & \boldmath{$\alpha_{kv}$} & \boldmath{$d_h$} & \boldmath{$d_h^R$} & \boldmath{$d_f$} \\
\midrule
2.9B & 2872.63M  & 24 & 24 & 2 & 1024 & 512 & 3072 & $\sqrt{3}$ & $\sqrt{12}$ & 128 & 64 & 10048 \\
\bottomrule
\end{tabular}
}
\label{tab:arch_gla_2_initialization}
\vspace{-0.1in}
\end{table}

\begin{table}[H]
\caption{Model configuration of GLA-4 for initialization ablation.}
\vspace{-0.1in}
\resizebox{\linewidth}{!}{
\centering
\setlength{\tabcolsep}{5pt}
\begin{tabular}{c | c c c c c c c c c c c c}
\toprule
\textbf{Model Size} & \textbf{\# Parameters} & \textbf{\# Layers} & \boldmath{$h$} & \boldmath{$g$} & \boldmath{$d_c^{\prime}$} & \boldmath{$d_c$} & \boldmath{$d$} & \boldmath{$\alpha_q$} & \boldmath{$\alpha_{kv}$} & \boldmath{$d_h$} & \boldmath{$d_h^R$} & \boldmath{$d_f$} \\
\midrule
2.9B & 2873.22M  & 24 & 24 & 4 & 1024 & 512 & 3072 & $\sqrt{3}$ & $\sqrt{24}$  & 128 & 64 & 10136 \\
\bottomrule
\end{tabular}
}
\label{tab:arch_gla_4_initialization}
\vspace{-0.1in}
\end{table}

\begin{table}[H]
\caption{Model configuration of GTA for initialization ablation.}
\vspace{-0.1in}
\resizebox{\linewidth}{!}{
\centering
\setlength{\tabcolsep}{10pt}
\begin{tabular}{c | c c c c c c c c}
\toprule
\textbf{Model Size} & \textbf{\# Parameters} & \textbf{\# Layers} & \boldmath{$h$} & \boldmath{$g$} & \boldmath{$d$}  & \boldmath{$d_h$} & \boldmath{$d_h^R$} & \boldmath{$d_f$}\\
\midrule
2.9B & 2872.00M  & 24 & 24 & 6 & 3072 & 128 & 64 & 9960 \\
\bottomrule
\end{tabular}
}
\label{tab:arch_gta_initialization}
\end{table}
\subsection{Architectural Hyperparameters for Scaling Ablation Study}
In our scaling ablation study, we investigate the impact of the scaling factors $\alpha_{q}$, $\alpha_{kv}$, and $\alpha_{attn}$ applied to the query latent states ($\bm{C}^{\text{Q}}$), the KV latent states ($\bm{C}^{\text{KV}}$), and the final attention output ($\tens{O}$), respectively. To determine the optimal configuration, we compare the model's performance with and without these scaling factors, where the `without' setting corresponds to fixing $\alpha_{q}$, $\alpha_{kv}$, and $\alpha_{attn}$ to 1. To isolate the impact of this scaling strategy, the model architecture and all other hyperparameters remain identical to those used in our main results. Detailed architectural specifications for these ablation experiments are provided in Tables~\ref{tab:arch_mla_scaling},~\ref{tab:arch_gla_2_scaling}, and~\ref{tab:arch_mlra_2_scaling}.
\begin{table}[H]
\caption{Model configuration of MLA in the absence of scaling factors.}
\vspace{-0.1in}
\resizebox{\linewidth}{!}{
\centering
\setlength{\tabcolsep}{5pt}
\begin{tabular}{c | c c c c c c c c c c c}
\toprule
\textbf{Model Size} & \textbf{\# Parameters} & \textbf{\# Layers} & \boldmath{$h$} & \boldmath{$d_c^{\prime}$}  & \boldmath{$d_c$} & \boldmath{$d$} & \boldmath{$\alpha_q$} & \boldmath{$\alpha_{kv}$} & \boldmath{$d_h$} & \boldmath{$d_h^R$} & \boldmath{$d_f$} \\
\midrule
2.9B & 2872.05M  & 24 & 24 & 1536 & 512 & 3072 & $1$ & $1$ & 128 & 64 & 9448 \\
\bottomrule
\end{tabular}
}
\label{tab:arch_mla_scaling}
\vspace{-0.1in}
\end{table}

\begin{table}[H]
\caption{Model configuration of GLA-2 in the absence of scaling factors.}
\vspace{-0.1in}
\resizebox{\linewidth}{!}{
\centering
\setlength{\tabcolsep}{3pt}
\begin{tabular}{c | c c c c c c c c c c c c}
\toprule
\textbf{Model Size} & \textbf{\# Parameters} & \textbf{\# Layers} & \boldmath{$h$} & \boldmath{$g$} & \boldmath{$d_c^{\prime}$} & \boldmath{$d_c$} & \boldmath{$d$} & \boldmath{$\alpha_q$} & \boldmath{$\alpha_{kv}$} & \boldmath{$d_h$} & \boldmath{$d_h^R$} & \boldmath{$d_f$} \\
\midrule
2.9B & 2872.63M  & 24 & 24 & 2 & 1024 & 512 & 3072 & $1$ & $1$ & 128 & 64 & 10048 \\
\bottomrule
\end{tabular}
}
\label{tab:arch_gla_2_scaling}
\vspace{-0.1in}
\end{table}

\begin{table}[H]
\caption{Model configuration of MLRA-2 in the absence of scaling factors.}
\vspace{-0.1in}
\resizebox{\linewidth}{!}{
\centering
\setlength{\tabcolsep}{3pt}
\begin{tabular}{c | c c c c c c c c c c c c}
\toprule
\textbf{Model Size} & \textbf{\# Parameters} & \textbf{\# Layers} & \boldmath{$h$} & \boldmath{$d_c^{\prime}$} & \boldmath{$d_c$} & \boldmath{$d$} & \boldmath{$\alpha_q$} & \boldmath{$\alpha_{kv}$} & \boldmath{$\alpha_{attn}$} & \boldmath{$d_h$} & \boldmath{$d_h^R$} & \boldmath{$d_f$}  \\
\midrule
2.9B & 2872.63M  & 24 & 24 & 1024 & 512 & 3072 & $1$ & $1$ & $1$  & 128 & 64 & 10048 \\
\bottomrule
\end{tabular}
}
\label{tab:arch_mlra_2_scaling}
\end{table}
\subsection{Architectural Hyperparameters for Double Heads Ablation Study}
\label{appendix:double_heads_parameters}
In our head-count ablation study, we investigate whether doubling the number of attention heads for GQA, MLA, and GLA-2 improves performance. Specifically, we increase the number of heads to 48 while maintaining the original KV cache size. To maintain parameter parity with our main results, we decrease the Feed-Forward Network (FFN) intermediate dimension. By keeping all other hyperparameters identical, we isolate the specific impact of the doubled head count. The detailed architectural specifications for these experiments are provided in Tables~\ref{tab:arch_gqa_double},~\ref{tab:arch_mla_double}, and~\ref{tab:arch_gla_2_double}.
\begin{table}[H]
\caption{Model configuration of GQA parameterized with $2 \times$ attention heads.}
\vspace{-0.1in}
\resizebox{\linewidth}{!}{
\centering
\setlength{\tabcolsep}{12pt}
\begin{tabular}{c | c c c c c c c}
\toprule
\textbf{Model Size} & \textbf{\# Parameters} & \textbf{\# Layers} & \boldmath{$h$} & \boldmath{$g$} & \boldmath{$d$}  & \boldmath{$d_h$} & \boldmath{$d_f$}\\
\midrule
2.9B & 2872.59M  & 24 & 48 & 6 & 3072 & 128 & 7680 \\
\bottomrule
\end{tabular}
}
\label{tab:arch_gqa_double}
\vspace{-0.1in}
\end{table}

\begin{table}[H]
\caption{Model configuration of MLA parameterized with $2 \times$ attention heads.}
\vspace{-0.1in}
\resizebox{\linewidth}{!}{
\centering
\setlength{\tabcolsep}{5pt}
\begin{tabular}{c | c c c c c c c c c c c}
\toprule
\textbf{Model Size} & \textbf{\# Parameters} & \textbf{\# Layers} & \boldmath{$h$} & \boldmath{$d_c^{\prime}$}  & \boldmath{$d_c$} & \boldmath{$d$} & \boldmath{$\alpha_q$} & \boldmath{$\alpha_{kv}$} & \boldmath{$d_h$} & \boldmath{$d_h^R$} & \boldmath{$d_f$} \\
\midrule
2.9B & 2873.23M  & 24 & 48 & 1536 & 512 & 3072 & $\sqrt{2}$ & $\sqrt{6}$ & 128 & 64 & 7320 \\
\bottomrule
\end{tabular}
}
\label{tab:arch_mla_double}
\vspace{-0.1in}
\end{table}

\begin{table}[H]
\caption{Model configuration of GLA-2 parameterized with $2 \times$ attention heads.}
\vspace{-0.1in}
\resizebox{\linewidth}{!}{
\centering
\setlength{\tabcolsep}{3pt}
\begin{tabular}{c | c c c c c c c c c c c c}
\toprule
\textbf{Model Size} & \textbf{\# Parameters} & \textbf{\# Layers} & \boldmath{$h$} & \boldmath{$g$} & \boldmath{$d_c^{\prime}$} & \boldmath{$d_c$} & \boldmath{$d$} & \boldmath{$\alpha_q$} & \boldmath{$\alpha_{kv}$} & \boldmath{$d_h$} & \boldmath{$d_h^R$} & \boldmath{$d_f$} \\
\midrule
2.9B & 2873.22M  & 24 & 48 & 2 & 1024 & 512 & 3072 & $\sqrt{3}$ & $\sqrt{12}$ & 128 & 64 & 8344 \\
\bottomrule
\end{tabular}
}
\label{tab:arch_gla_2_double}
\end{table}

\subsection{Architectural Hyperparameters for Gated Attention Study}
\label{appendix:gated_parameters}
In our gated attention study, we investigate whether incorporating a gating mechanism~\citep{hochreiter1997long,srivastava2015highway,dey2017gate,qiu2025gated} into GQA, MLA, GLA-2, MLRA-2, and MLRA-4 improves performance. Specifically, we integrate gated attention into these architectures as shown in Appendix~\ref{appendix:gated_attention}. To maintain parameter parity with our main results, we proportionally decrease the Feed-Forward Network (FFN) intermediate dimension to offset the additional gate parameters. By keeping all other hyperparameters identical, we isolate the specific impact of the gating strategy. The detailed architectural specifications for these experiments are provided in Tables~\ref{tab:arch_gqa_gated},~\ref{tab:arch_mla_gated},~\ref{tab:arch_gla_2_gated},~\ref{tab:arch_mlra_2_gated}, and~\ref{tab:arch_mlra_4_gated}.
\begin{table}[H]
\caption{Model configuration of GQA incorporating gated attention.}
\vspace{-0.1in}
\resizebox{\linewidth}{!}{
\centering
\setlength{\tabcolsep}{12pt}
\begin{tabular}{c | c c c c c c c}
\toprule
\textbf{Model Size} & \textbf{\# Parameters} & \textbf{\# Layers} & \boldmath{$h$} & \boldmath{$g$} & \boldmath{$d$}  & \boldmath{$d_h$} & \boldmath{$d_f$}\\
\midrule
2.9B & 2872.59M  & 24 & 24 & 6 & 3072 & 128 & 8704 \\
\bottomrule
\end{tabular}
}
\label{tab:arch_gqa_gated}
\vspace{-0.1in}
\end{table}

\begin{table}[H]
\caption{Model configuration of MLA incorporating gated attention.}
\vspace{-0.1in}
\resizebox{\linewidth}{!}{
\centering
\setlength{\tabcolsep}{5pt}
\begin{tabular}{c | c c c c c c c c c c c}
\toprule
\textbf{Model Size} & \textbf{\# Parameters} & \textbf{\# Layers} & \boldmath{$h$} & \boldmath{$d_c^{\prime}$}  & \boldmath{$d_c$} & \boldmath{$d$} & \boldmath{$\alpha_q$} & \boldmath{$\alpha_{kv}$} & \boldmath{$d_h$} & \boldmath{$d_h^R$} & \boldmath{$d_f$} \\
\midrule
2.9B & 2872.05M  & 24 & 24 & 1536 & 512 & 3072 & $\sqrt{2}$ & $\sqrt{6}$ & 128 & 64 & 8424 \\
\bottomrule
\end{tabular}
}
\label{tab:arch_mla_gated}
\vspace{-0.1in}
\end{table}

\begin{table}[H]
\caption{Model configuration of GLA-2 incorporating gated attention.}
\vspace{-0.1in}
\resizebox{\linewidth}{!}{
\centering
\setlength{\tabcolsep}{3pt}
\begin{tabular}{c | c c c c c c c c c c c c}
\toprule
\textbf{Model Size} & \textbf{\# Parameters} & \textbf{\# Layers} & \boldmath{$h$} & \boldmath{$g$} & \boldmath{$d_c^{\prime}$} & \boldmath{$d_c$} & \boldmath{$d$} & \boldmath{$\alpha_q$} & \boldmath{$\alpha_{kv}$} & \boldmath{$d_h$} & \boldmath{$d_h^R$} & \boldmath{$d_f$} \\
\midrule
2.9B & 2872.63M  & 24 & 24 & 2 & 1024 & 512 & 3072 & $\sqrt{3}$ & $\sqrt{12}$ & 128 & 64 & 9024 \\
\bottomrule
\end{tabular}
}
\label{tab:arch_gla_2_gated}
\vspace{-0.1in}
\end{table}

\begin{table}[H]
\caption{Model configuration of MLRA-2 incorporating gated attention.}
\vspace{-0.1in}
\resizebox{\linewidth}{!}{
\centering
\setlength{\tabcolsep}{3pt}
\begin{tabular}{c | c c c c c c c c c c c c}
\toprule
\textbf{Model Size} & \textbf{\# Parameters} & \textbf{\# Layers} & \boldmath{$h$} & \boldmath{$d_c^{\prime}$} & \boldmath{$d_c$} & \boldmath{$d$} & \boldmath{$\alpha_q$} & \boldmath{$\alpha_{kv}$} & \boldmath{$\alpha_{attn}$} & \boldmath{$d_h$} & \boldmath{$d_h^R$} & \boldmath{$d_f$}  \\
\midrule
2.9B & 2872.63M  & 24 & 24 & 1024 & 512 & 3072 & $\sqrt{3}$ & $\sqrt{24}$ & $\frac{\sqrt{2}}{2}$  & 128 & 64 & 9024 \\
\bottomrule
\end{tabular}
}
\label{tab:arch_mlra_2_gated}
\vspace{-0.1in}
\end{table}

\begin{table}[H]
\caption{Model configuration of MLRA-4 incorporating gated attention.}
\vspace{-0.1in}
\resizebox{\linewidth}{!}{
\centering
\setlength{\tabcolsep}{3pt}
\begin{tabular}{c | c c c c c c c c c c c c}
\toprule
\textbf{Model Size} & \textbf{\# Parameters} & \textbf{\# Layers} & \boldmath{$h$} & \boldmath{$d_c^{\prime}$} & \boldmath{$d_c$} & \boldmath{$d$} & \boldmath{$\alpha_q$} & \boldmath{$\alpha_{kv}$} & \boldmath{$\alpha_{attn}$} & \boldmath{$d_h$} & \boldmath{$d_h^R$} & \boldmath{$d_f$}   \\
\midrule
2.9B & 2873.22M  & 24 & 24 & 1024 & 512 & 3072 & $\sqrt{3}$ & $\sqrt{24}$ & $\frac{1}{2}$ & 128 & 64 & 8856  \\
\bottomrule
\end{tabular}
}
\label{tab:arch_mlra_4_gated}
\end{table}

\section{Additional Experimental Results}
\begin{table}[H]
\caption{Validation perplexity (lower is better) across seven datasets: Wikipedia, C4, Pile, RefinedWeb, Cosmopedia, FineWeb, and FineWeb-Edu.  We compare two initialization strategies, zero versus Gaussian ($\mathcal{N}(0, \sigma=0.02)$), applied to the output projection weights $\bm{W}^{\text{O, attn}}$ and $\bm{W}^{\text{O, mlp}}$.}
\vspace{-0.1in}
\resizebox{\linewidth}{!}{
\centering
\setlength{\tabcolsep}{5pt}
\begin{tabular}{c | c | c c c c c c c | c}
\toprule
\textbf{Method} & \textbf{Initialization} & \textbf{Wikipedia} & \textbf{C4} & \textbf{Pile} & \textbf{RefinedWeb} & \textbf{Cosmopedia} & \textbf{FineWeb} & \textbf{FineWeb-Edu} & \textbf{Avg} \\
\midrule
MHA & $\mathcal{N}(0,\sigma=0.02)$ & 14.759 & 16.800 & 13.282 & 18.988 & 9.356 & 15.904 & 9.571 & 14.094 \\
MHA & zero                         & 14.624 & 16.575 & 12.929 & 18.698 & 9.102 & 15.656 & 9.434 & 13.860 \\
\midrule
MQA & $\mathcal{N}(0,\sigma=0.02)$ & 14.708 & 17.075 & 13.500 & 19.301 & 9.510 & 16.190 & 9.697 & 14.283 \\
MQA & zero                         & 15.134 & 16.837 & 14.008 & 19.202 & 9.484 & 15.942 & 9.533 & 14.306 \\
\midrule
GQA & $\mathcal{N}(0,\sigma=0.02)$ & 14.687 & 16.882 & 13.528 & 19.084 & 9.422 & 15.974 & 9.571 & 14.164 \\
GQA & zero                         & 15.057 & 16.628 & 13.758 & 18.885 & 9.504 & 15.713 & 9.427 & 14.139 \\
\midrule
MLA & $\mathcal{N}(0,\sigma=0.02)$ & 14.571 & 16.624 & 13.113 & 18.837 & 9.110 & 15.740 & 9.490 & 13.927 \\
MLA & zero                         & 14.567 & 16.345 & 12.965 & 18.523 & 8.966 & 15.440 & 9.284 & 13.727 \\
\midrule
MFA & $\mathcal{N}(0,\sigma=0.02)$ & 15.123 & 17.032 & 13.752 & 19.133 & 9.550 & 16.138 & 9.707 & 14.374 \\
MFA & zero                         & 15.693 & 16.738 & 13.903 & 19.125 & 9.423 & 15.815 & 9.506 & 14.315 \\
\midrule
TPA & $\mathcal{N}(0,\sigma=0.02)$ & 15.205 & 17.128 & 13.814 & 19.445 & 9.844 & 16.227 & 9.682 & 14.478 \\
TPA & zero                         & 14.789 & 16.622 & 13.333 & 18.971 & 9.130 & 15.717 & 9.333 & 13.985 \\
\midrule
GLA-2 & $\mathcal{N}(0,\sigma=0.02)$ & 14.717 & 16.675 & 13.216 & 18.886 & 9.259 & 15.799 & 9.510 & 14.009 \\
GLA-2 & zero                         & 14.605 & 16.323 & 13.225 & 18.509 & 9.118 & 15.424 & 9.249 & 13.779 \\
\midrule
GLA-4 & $\mathcal{N}(0,\sigma=0.02)$ & 14.858 & 16.791 & 13.522 & 18.953 & 9.374 & 15.914 & 9.571 & 14.140 \\
GLA-4 & zero                         & 14.547 & 16.436 & 13.229 & 18.578 & 9.076 & 15.535 & 9.307 & 13.815 \\
\midrule
GTA & $\mathcal{N}(0,\sigma=0.02)$ & 14.896 & 16.959 & 13.621 & 19.277 & 9.536 & 16.061 & 9.647 & 14.285 \\
GTA & zero                         & 14.733 & 16.599 & 13.402 & 18.924 & 9.129 & 15.672 & 9.346 & 13.972 \\
\bottomrule
\end{tabular}
}
\label{tab:ablation_initialization}
\end{table}

\begin{table}[H]
\caption{Validation perplexity (lower is better) across seven datasets: Wikipedia, C4, Pile, RefinedWeb, Cosmopedia, FineWeb, and FineWeb-Edu.  This analysis specifically compares models without and with scaling.}
\vspace{-0.1in}
\resizebox{\linewidth}{!}{
\centering
\setlength{\tabcolsep}{5pt}
\begin{tabular}{c | c | c c c c c c c | c}
\toprule
\textbf{Method} & \textbf{Scaling} & \textbf{Wikipedia} & \textbf{C4} & \textbf{Pile} & \textbf{RefinedWeb} & \textbf{Cosmopedia} & \textbf{FineWeb} & \textbf{FineWeb-Edu} & \textbf{Avg} \\
\midrule
MLA & w/o & 14.461 & 16.386 & 13.218 & 18.636 & 8.961 & 15.485 & 9.307 & 13.779 \\
MLA & w/  & 14.567 & 16.345 & 12.965 & 18.523 & 8.966 & 15.440 & 9.284 & 13.727 \\
\midrule
GLA-2 & w/o & 14.518 & 16.467 & 13.179 & 18.612 & 9.138 & 15.565 & 9.305 & 13.827 \\
GLA-2 & w/  & 14.605 & 16.323 & 13.225 & 18.509 & 9.118 & 15.424 & 9.249 & 13.779 \\
\midrule
MLRA-2 & w/o & 14.326 & 16.485 & 13.145 & 18.657 & 9.168 & 15.570 & 9.304 & 13.808 \\
MLRA-2 & w/  & 14.615 & 16.342 & 13.236 & 18.602 & 9.153 & 15.439 & 9.242 & 13.804 \\
\bottomrule
\end{tabular}
}
\label{tab:ablation_scale}
\end{table}

\begin{table}[H]
\caption{Validation perplexity (lower is better) across seven datasets: Wikipedia, C4, Pile, RefinedWeb, Cosmopedia, FineWeb, and FineWeb-Edu. This analysis specifically compares models with and without $2 \times$ attention heads.}
\vspace{-0.1in}
\resizebox{\linewidth}{!}{
\centering
\setlength{\tabcolsep}{5pt}
\begin{tabular}{c | c | c c c c c c c | c}
\toprule
\textbf{Method} & \textbf{$2 \times$ Attention Heads} & \textbf{Wikipedia} & \textbf{C4} & \textbf{Pile} & \textbf{RefinedWeb} & \textbf{Cosmopedia} & \textbf{FineWeb} & \textbf{FineWeb-Edu} & \textbf{Avg} \\
\midrule
GQA & w/  & 15.280 & 16.702 & 13.789 & 18.961 & 9.486 & 15.785 & 9.490 & 14.213 \\
GQA & w/o & 15.057 & 16.628 & 13.758 & 18.885 & 9.504 & 15.713 & 9.427 & 14.139 \\
\midrule
MLA & w/  & 14.771 & 16.432 & 13.108 & 18.615 & 9.029 & 15.529 & 9.371 & 13.836 \\
MLA & w/o & 14.567 & 16.345 & 12.965 & 18.523 & 8.966 & 15.440 & 9.284 & 13.727 \\
\midrule
GLA-2 & w/  & 14.969 & 16.313 & 13.428 & 18.569 & 8.991 & 15.410 & 9.281 & 13.851 \\
GLA-2 & w/o & 14.605 & 16.323 & 13.225 & 18.509 & 9.118 & 15.424 & 9.249 & 13.779 \\
\bottomrule
\end{tabular}
}
\label{tab:ablation_double_heads}
\end{table}

\section{Illustration}
\begin{figure}[H]
    \centering
        \centering
        \includegraphics[width=0.479\textwidth]{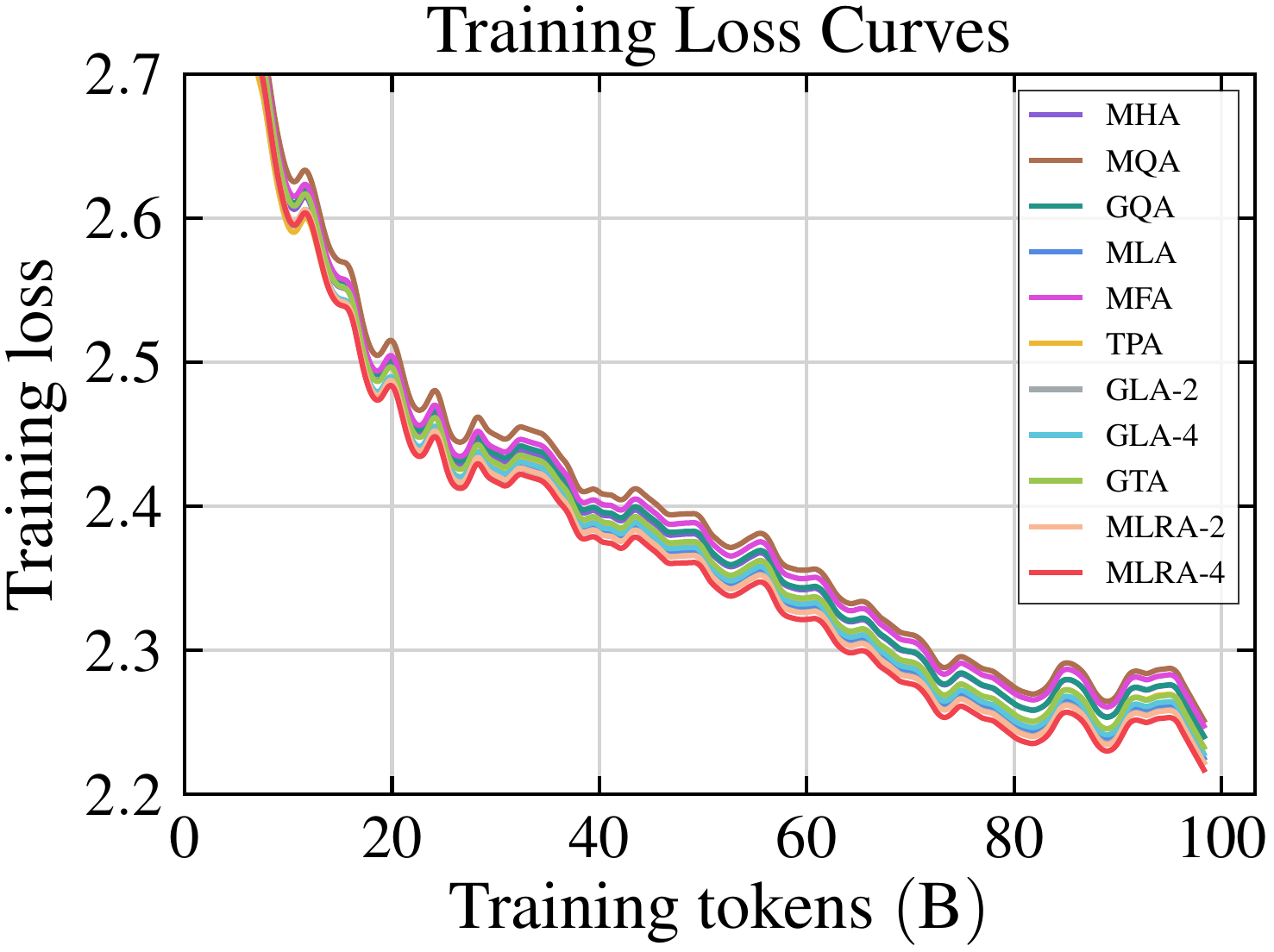}
    \caption{Training loss curves for all models.}
    \label{fig:loss}
\end{figure}
\begin{figure}[H]
    \centering
        \centering
        \includegraphics[width=\textwidth]{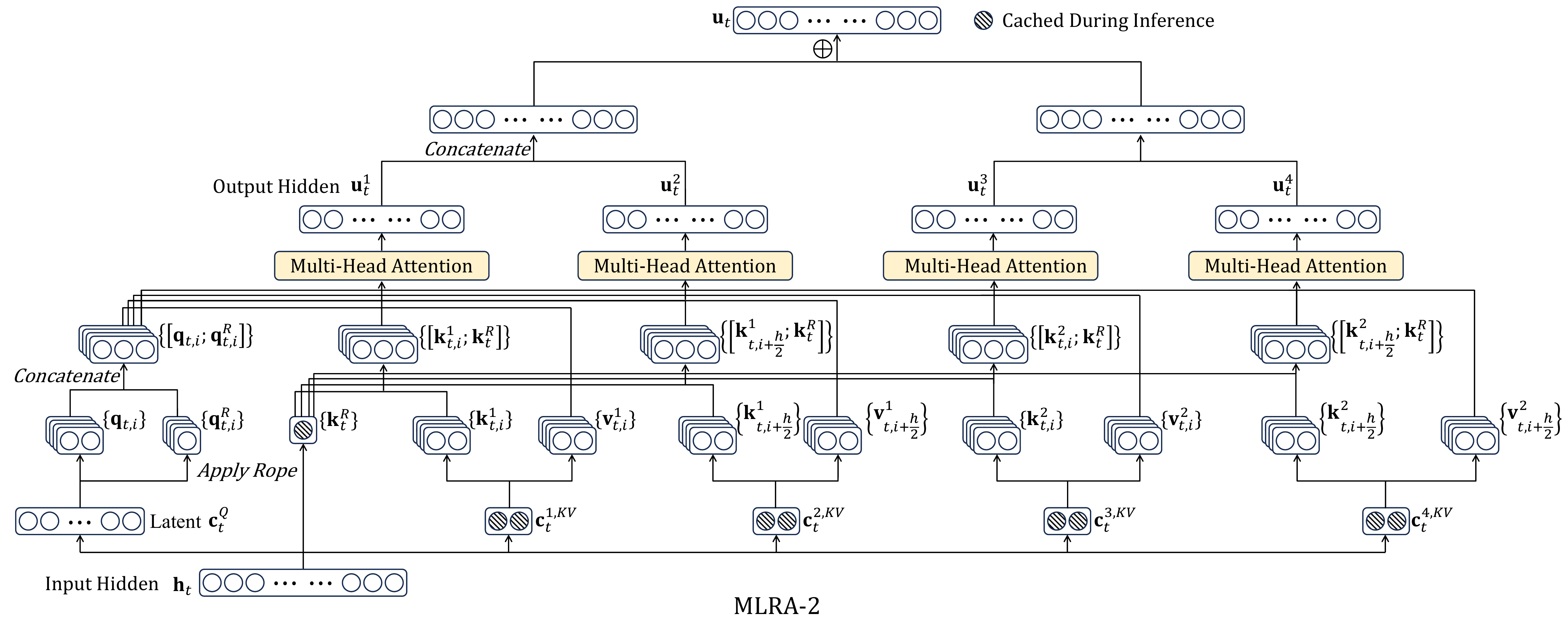}
        \vspace{-0.15in}
    \caption{Illustration of MLRA-2.}
    \label{fig:mlra_2}
\end{figure}
\begin{figure}[H]
    \centering
        \centering
        \includegraphics[width=\textwidth]{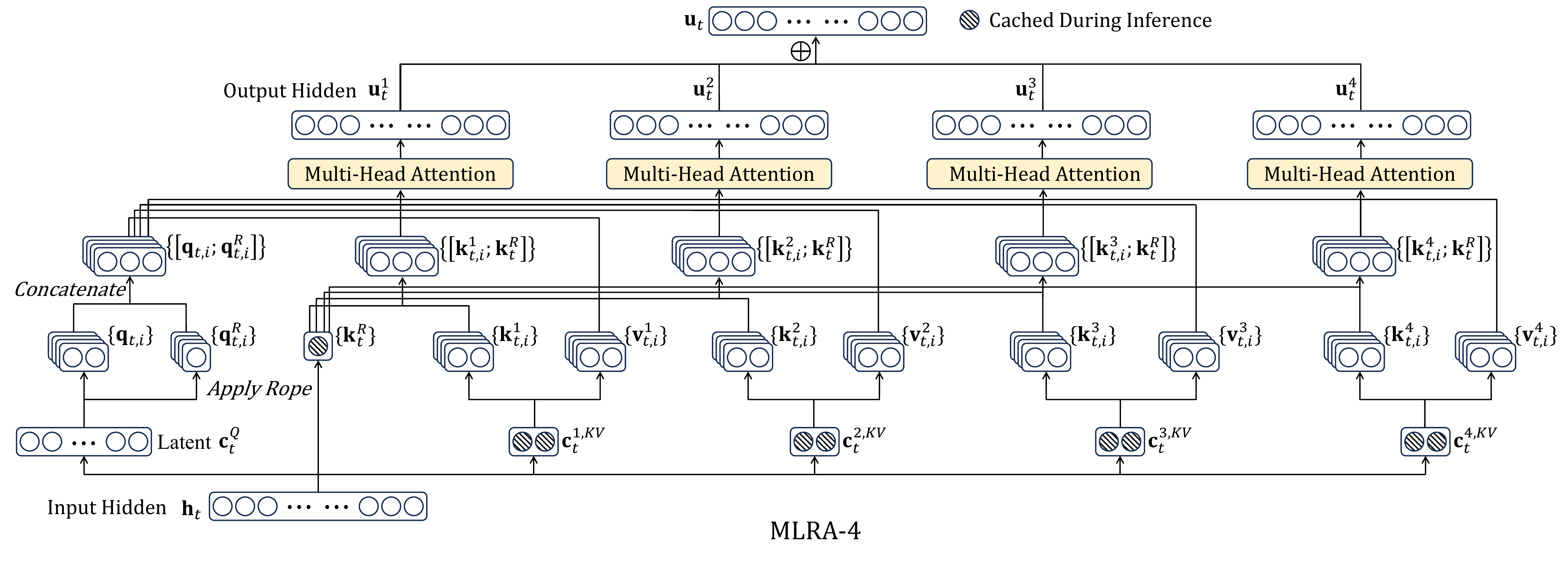}
        \vspace{-0.15in}
    \caption{Illustration of MLRA-4.}
    \label{fig:mlra_4}
\end{figure}
\section{Related Work}
\paragraph{KV Cache Compression.} 
Recent works~\citep{liu2023scissorhands,anagnostidis2023dynamic,zhang2023h2o,ge2024model,xiao2024efficient,kim2024compressed,zhang2024cam,nawrot2024dynamic,tang2024quest,liu24kivi,dong2024get,cai2024lococo,liu2024cachegen,hooper2024kvquant,sun2024you,chen2024arkvale,jiang2024minference,li2024snapkv,xiao2025duoattention,sun2025shadowkv,meng2025transmla,tang2025tpla} don't introduce new attention mechanisms; instead, they compress the KV cache for pretrained models. Some of these works~\citep{liu24kivi,hooper2024kvquant} use quantization to store the KV cache in low-bit formats. Some other approaches~\citep{zhang2023h2o,xiao2024efficient,li2024snapkv,xiao2025duoattention} retain important tokens and discard others to compress the KV cache.

\paragraph{Low-Rank Approximation.} 
Low-rank approximation~\citep{hu2022lora,malladi2023kernel,zhang2023adaptive,dettmers2023qlora,lialin2024relora,zhu2024asymmetry,zeng2024the,chen2024longlora,lin2025modegpt,wang2025svdllm,chang2025palu} are widely used to compress representations to a low-dimensional space, then up-project to recover full representations. These methods greatly reduce trainable parameters~\citep{hu2022lora,dettmers2023qlora} during fine-tuning and decrease the number of parameters~\citep{lin2025modegpt,wang2025svdllm} for pretrained models.

\paragraph{System for Attention.} FlashAttention~\citep{dao2022flashattention,dao2023flashattention2,shah2024flashattention} uses tiling and online softmax to minimize reads and writes between high-bandwidth memory and on-chip SRAM, shifting attention from a memory bottleneck to a compute bottleneck. FlashMLA~\citep{flashmla2025} avoids explicit KV materialization during attention decoding by absorbing the key up-projection matrices into the queries. The following attention computation is similar to MQA with shared KV states. Inspired by classical virtual memory and paging in operating systems, PagedAttention~\citep{kwon2023efficient} and vLLM use block-level memory management and preemptive request scheduling to reduce fragmentation and redundant duplication.

\paragraph{Linear Attention.} 
Linear attention~\citep{katharopoulos2020transformers,peng2021random,schlag2021linear,gu2022efficiently,smith2023simplified,sun2023retentive,qin2023hierarchically,yang2024gated,dao2024transformers,peng2024eagle,gu2024mamba,beck2024xlstm,zhang2024gated,yang2024parallelizing,yang2025gated} reformulates the attention mechanism by substituting the exponential kernel in softmax with a dot product between the query and key vectors. It reduces the memory complexity per decoding step from $\mathcal{O}\left(n\right)$ for full attention to $\mathcal{O}\left(1\right)$.

\end{document}